\documentclass[11pt]{article}
\usepackage[a4paper,margin=25mm]{geometry}
\usepackage[T1]{fontenc}
\usepackage{amsmath,amssymb,amsthm,mathtools}
\usepackage{newtxtext,newtxmath}
\usepackage{microtype,booktabs,array}
\usepackage{algorithm,algpseudocode}
\usepackage[round,authoryear]{natbib}
\usepackage{xcolor}
\usepackage[colorlinks=true,linkcolor=blue!45!black,citecolor=blue!45!black,urlcolor=blue!45!black]{hyperref}
\usepackage{bookmark}
\usepackage{needspace}
\hypersetup{pdftitle={Last-Iterate Convergence of Clipped SGD with Linear Decay}}
\newtheorem{assumption}{Assumption}
\newtheorem{theorem}{Theorem}
\newtheorem{lemma}{Lemma}
\newtheorem{proposition}{Proposition}
\newtheorem{corollary}{Corollary}
\theoremstyle{remark}

\newcommand{\E}{\mathbb E}
\newcommand{\Prob}{\mathbb P}
\newcommand{\R}{\mathbb R}
\newcommand{\F}{\mathcal F}
\newcommand{\Et}{\mathbb E_t}
\newcommand{\clip}{\operatorname{clip}}
\newcommand{\prox}{\operatorname{prox}}
\newcommand{\norm}[1]{\lVert #1\rVert}
\newcommand{\ip}[2]{\langle #1,#2\rangle}
\newcommand{\ind}{\mathbf 1}
\newcommand{\dom}{\operatorname{dom}}

\allowdisplaybreaks[1]
\makeatletter
\renewcommand{\@maketitle}{%
  \newpage\null\vskip .5em
  \begin{center}
    {\LARGE\@title\par}
    \ifx\@author\@empty\else
      \vskip 1em{\large\begin{tabular}[t]{c}\@author\end{tabular}\par}
    \fi
    \ifx\@date\@empty\else\vskip .5em{\large\@date\par}\fi
  \end{center}\par\vskip 1em}
\makeatother
\title{Last-Iterate Convergence of Clipped SGD with Linear Decay}
\author{
{Wei Jiang}$^{1,2}$ \quad
{Haomin Bai}$^{2}$ \quad
{Zechao Li}$^{1}$ \\[0.5em]
\small$^1$School of Computer Science and Engineering, 
Nanjing University of Science and Technology\\
\small$^2$National Key Laboratory for Novel Software Technology, Nanjing University\\
\small \{jiangw, baihm\}@lamda.nju.edu.cn, \quad zechao.li@njust.edu.cn
}
\date{}
\begin{document}
\maketitle
\begin{abstract}
In this paper, we study the last-iterate convergence of clipped stochastic gradient descent under heavy-tailed noise. We assume that the stochastic gradients are unbiased and that the noise has a bounded central $p$-th moment, where $1<p\le2$. For smooth convex objectives, a linearly decreasing learning rate gives a high-probability error bound of order $O(LR^2\log(1/\delta)/T+\sigma R(\log(1/\delta)/T)^{(p-1)/p})$. We further establish last-iterate guarantees for nonsmooth composite objectives and smooth composite objectives with proximal updates. In addition, when a causal adversary replaces at most $q$ gradients, the same clipped update achieves a bound with $\log(1/\delta)$ replaced by $\log(1/\delta)+q$. The corresponding lower bounds for the clean and corrupted oracle models are also provided.
\end{abstract}

\section{Introduction}\label{sec:main}
In this paper, we study the problem
\begin{equation}\label{eq:problem}
\min_{x\in\R^d} F(x):=f(x)+r(x).
\end{equation}
When $r=0$, the problem reduces to the smooth unconstrained problem. For such a problem, gradient clipping allows stochastic optimization methods to handle noise whose variance may be infinite. For convex objectives, high-probability guarantees with the statistical rate $T^{-(p-1)/p}$ are available for averaged outputs under a bounded central $p$-th moment \citep{nguyen2023,liu2026refined}. Obtaining this rate for the last iterate of clipped SGD requires a different analysis, since a bound on the average objective value does not control the objective value after the final update. 

This paper studies last-iterate convergence of clipped SGD algorithm. We first establish a last-iterate guarantee for smooth convex objectives. Then, we investigate proximal updates for nonsmooth and smooth composite objectives, seperately. Finally, we  analyze the setting in which a number of samples can be replaced by a causal adversary. Table~\ref{tab:results} summarizes these bounds.

\begin{table}[t]
\centering\small
\caption{High-probability guarantees for the final update. Here $a=(p-1)/p$, $\ell\asymp\log(1/\delta)$.}\label{tab:results}
\begin{tabular}{@{}>{\raggedright\arraybackslash}p{.38\textwidth}p{.38\textwidth}p{.15\textwidth}@{}}
\toprule
Setting & Objective error & Oracle calls\\
\midrule
Smooth, Theorem~\ref{thm:main} & $LR^2\ell/T+\sigma R(\ell/T)^a$ & $T$\\[3pt]
Nonsmooth composite, Theorem~\ref{thm:nonsmooth} & $GR/\sqrt T+GR\ell/T+\sigma R(\ell/T)^a$ & $T$\\[3pt]
Smooth composite, Theorem~\ref{thm:composite} & $LR^2\ell/T+\sigma R(\ell/T)^a$ & $T+O(\ell)$\\[3pt]
$q$ replacements, Theorem~\ref{thm:corrupt} & $LR^2(\ell+q)/T+\sigma R((\ell+q)/T)^a$ & $T$\\
\bottomrule
\end{tabular}
\end{table}
\section{Theoretical guarantee}
We first introduce the assumptions used.
\begin{assumption}\label{ass:convex}
Function $f:\R^d\to\R$ is convex, and $r:\R^d\to\R\cup\{+\infty\}$ is proper, lower semicontinuous, and convex. There exists $x^\star\in\arg\min F$ and we denote $F^\star=F(x^\star)$. The initial point satisfies $x_1\in\dom r$ and $\norm{x_1-x^\star}\le R$, for a known $R>0$. 
\end{assumption}

\begin{assumption}\label{ass:noise}
For some $1<p\le2$ and $\sigma\ge0$, the vector $g_t$ satisfies, almost surely,
\begin{equation}\label{eq:oracle}
h_t:=\Et[g_t]\in\partial f(x_t),\qquad
\Et[\norm{g_t-h_t}^p]\le\sigma^p.
\end{equation}
\end{assumption}
\noindent When $f$ is differentiable, $h_t=\nabla f(x_t)$. 
\begin{assumption}\label{ass:smooth}
Function $f$ has an $L$-Lipschitz gradient:
\[
\norm{\nabla f(x)-\nabla f(y)}\le L\norm{x-y},\quad L>0,\quad x,y\in\R^d.
\]
\end{assumption}

\noindent For $\lambda>0$, define
\[
\clip(g,\lambda)=g\min\{1,\lambda/\norm g\},\qquad \clip(0,\lambda)=0.
\]
Clipped SGD performs one oracle call per iteration, as described in Algorithm~\ref{alg:clip}. Its output is $x_{T+1}$.
\begin{algorithm}[t]
\caption{Clipped SGD with linear decay}\label{alg:clip}
\begin{algorithmic}[1]
\Require Initial point $x_1$, iteration number $T$, clipping level $\lambda>0$, base learning rate $\eta>0$
\For{$t=1,\ldots,T$}
\State Query the stochastic gradient $g_t$ at $x_t$
\State Set $\eta_t=\eta(T-t+1)/T$ and $\widehat g_t=\clip(g_t,\lambda)$
\State Update $x_{t+1}=x_t-\eta_t\widehat g_t$
\EndFor
\State \Return $x_{T+1}$
\end{algorithmic}
\end{algorithm}
The learning rate decreases from $\eta$ to $\eta/T$. We use a fixed clipping level $\lambda$ throughout the run. The two terms in the $\lambda$ below have separate roles: $4LR$ ensures that the true gradient remains below the clipping level, while $\sigma(T/\ell)^{1/p}$ controls the accumulated clipping bias and noise.

\begin{theorem}\label{thm:main}
Under Assumptions~\ref{ass:convex}--\ref{ass:smooth} with $r=0$, fix an integer $T\ge1$ and $\delta\in(0,1)$. Run Algorithm~\ref{alg:clip} with
\begin{equation}\label{eq:parameters}
\ell=\log\frac8\delta,\qquad
\lambda=\max\left\{4LR,\sigma(T/\ell)^{1/p}\right\},\qquad
\eta=\frac{R}{128\lambda\ell}.
\end{equation}
With probability at least $1-\delta$,
\begin{align}
f(x_{T+1})-f(x^\star)
&\le\frac{256R\lambda\ell}{T}\label{eq:mainmax}\\
&\le1024\frac{LR^2\ell}{T}
+256\sigma R\left(\frac\ell T\right)^{(p-1)/p}.\label{eq:mainrate}
\end{align}
\end{theorem}

For $p=2$, the dominant term in \eqref{eq:mainrate} is $O(\sigma R\sqrt{\log(1/\delta)/T})$. The result needs neither a bounded feasible set nor a uniform bound on $\norm{\nabla f(x)}$. Proposition~\ref{prop:lower} gives the corresponding lower bound.

\section{Clipped SGD with proximal updates}\label{sec:extensions}
\subsection{Nonsmooth composite objectives}\label{sec:nonsmooth}
We now replace smoothness by the following condition.
\begin{assumption}[Lipschitz continuity]\label{ass:lipschitz}
For a known $G\ge0$, $|f(x)-f(y)|\le G\norm{x-y}$ for all $x,y\in\R^d$.
\end{assumption}
Every subgradient of $f$ then has norm at most $G$. The regularizer $r$ need not be Lipschitz. We assume exact access to
\[
\prox_{\alpha r}(y)=\arg\min_{u\in\R^d}
\left\{r(u)+\frac{\norm{u-y}^2}{2\alpha}\right\},\qquad \alpha>0.
\]
We use the proximal clipped update
\begin{equation}\label{eq:nonsmoothupdate}
x_{t+1}=\prox_{\eta_t r}\bigl(x_t-\eta_t\clip(g_t,\lambda)\bigr),
\qquad \eta_t=\eta\frac{T-t+1}{T}.
\end{equation}

\begin{theorem}\label{thm:nonsmooth}
Under Assumptions~\ref{ass:convex}, \ref{ass:noise}, and \ref{ass:lipschitz}, suppose $G+\sigma>0$. For an integer $T\ge1$ and $\delta\in(0,1)$, set
\begin{equation}\label{eq:nsparameters}
\ell=\log\frac4\delta,\quad
\lambda=\max\{2G,\sigma(T/\ell)^{1/p}\},\quad
H=G\sqrt T+\lambda\ell,\quad \eta=\frac{R}{128H}.
\end{equation}
After $T$ updates of \eqref{eq:nonsmoothupdate}, with probability at least $1-\delta$,
\begin{align}
F(x_{T+1})-F^\star
&\le\frac{256RH}{T}\label{eq:nsmax}\\
&\le256\frac{GR}{\sqrt T}+512\frac{GR\ell}{T}
+256\sigma R(\ell/T)^{(p-1)/p}.\label{eq:nsrate}
\end{align}
If $T\ge\ell^2$, the right-hand side is at most
$768GR/\sqrt T+256\sigma R(\ell/T)^{(p-1)/p}$.
\end{theorem}

\subsection{Smooth composite objectives}\label{sec:composite}
At $x_1$, query $m=\lceil32\log(2/\delta)\rceil$ pilot gradients $g^{(0)}_1,\ldots,g^{(0)}_m$, and let
\begin{equation}\label{eq:anchor}
c\in\arg\min_{z\in\R^d}\sum_{j=1}^m\norm{z-g^{(0)}_j}.
\end{equation}
We use a measurable choice of this geometric median. The pilot gradients satisfy the conditional oracle assumption at $x_1$, and the main-stage filtration includes all pilot data. The updates are
\begin{equation}\label{eq:compositeupdate}
\widehat g_t=c+\clip(g_t-c,\lambda),\qquad
x_{t+1}=\prox_{\eta_t r}(x_t-\eta_t\widehat g_t).
\end{equation}

\Needspace{17\baselineskip}
\begin{theorem}\label{thm:composite}
Under Assumptions~\ref{ass:convex}--\ref{ass:smooth}, fix $\delta\in(0,1)$ and construct $c$ by \eqref{eq:anchor}. For an integer $T\ge1$, choose
\begin{equation}\label{eq:compositeparameters}
\ell=\log\frac{16}\delta,\quad
\lambda=\max\{6LR+48\sigma,\sigma(T/\ell)^{1/p}\},\quad
\eta_t=\frac{R}{128\lambda\ell}\frac{T-t+1}{T}.
\end{equation}
After $T$ updates of \eqref{eq:compositeupdate}, with probability at least $1-\delta$,
\begin{equation}\label{eq:compositerate}
F(x_{T+1})-F^\star\le\frac{256R\lambda\ell}{T},\qquad
\max_{1\le t\le T+1}\norm{x_t-x^\star}\le2R.
\end{equation}
When $T\ge\ell$,
\begin{equation}\label{eq:compositesimplified}
F(x_{T+1})-F^\star
\le1536\frac{LR^2\ell}{T}+12544\sigma R(\ell/T)^{(p-1)/p}.
\end{equation}
The total number of stochastic-gradient calls is $N=T+m$.
\end{theorem}

\section{Adversarially corrupted gradients}\label{sec:corrupt}
Return to the smooth setting with $r=0$. After $x_t$ is selected, the oracle generates a latent clean gradient $g_t$ satisfying Assumption~\ref{ass:noise}. A causal adversary may observe $g_t$ and replace it by any vector $\widetilde g_t$. Let
\begin{equation}\label{eq:corruptmodel}
C_t=\ind_{\{\widetilde g_t\ne g_t\}},\qquad
\sum_{t=1}^T C_t\le q\quad\text{almost surely},
\end{equation}
where $q\in\{0,\ldots,T\}$ is a known upper bound. The adversary can depend on the past and the current clean sample, but not on future clean randomness. The filtration in Assumption~\ref{ass:noise} contains all past latent samples and adversarial decisions. The algorithm observes only $\widetilde g_t$.

\begin{theorem}\label{thm:corrupt}
Under Assumptions~\ref{ass:convex}--\ref{ass:smooth} with $r=0$ and the corruption model \eqref{eq:corruptmodel}, fix an integer $T\ge1$ and $\delta\in(0,1)$. Set
\begin{equation}\label{eq:corruptparameters}
\ell=\log\frac8\delta,\quad \Gamma=\ell+q,\quad
\lambda=\max\{4LR,\sigma(T/\Gamma)^{1/p}\},\quad
\eta=\frac{R}{128\lambda\Gamma}.
\end{equation}
Run Algorithm~\ref{alg:clip} using $\widetilde g_t$ in place of $g_t$. With probability at least $1-\delta$,
\begin{align}
f(x_{T+1})-f(x^\star)
&\le\frac{256R\lambda\Gamma}{T}\label{eq:corruptmax}\\
&\le1024LR^2\frac{\ell+q}{T}
+256\sigma R\left(\frac{\ell+q}{T}\right)^{(p-1)/p}.\label{eq:corruptrate}
\end{align}
On the same event, $\norm{x_t-x^\star}\le2R$ for $1\le t\le T+1$.
\end{theorem}

If $q\le\varepsilon T$, the two terms in \eqref{eq:corruptrate} have orders $LR^2(\ell/T+\varepsilon)$ and $\sigma R(\ell/T+\varepsilon)^{(p-1)/p}$. The model allows the adversary to select samples after inspecting them.

\begin{corollary}\label{cor:budget}
In Theorem~\ref{thm:corrupt}, replace the uniform moment condition by
$\Et\norm{g_t-\nabla f(x_t)}^p\le\nu_t$, where $\nu_t\ge0$ is $\F_{t-1}$-measurable and $\sum_{t=1}^T\nu_t\le S$ almost surely for a known deterministic $S\ge0$. With $\lambda=\max\{4LR,(S/\Gamma)^{1/p}\}$ and the same learning-rate formula,
\[
f(x_{T+1})-f(x^\star)
\le1024LR^2\frac\Gamma T+256R\frac{S^{1/p}\Gamma^{(p-1)/p}}T
\]
with probability at least $1-\delta$.
\end{corollary}

\section{Statistical lower bounds}\label{sec:lower}
We give three two-instance constructions to clarify the statistical terms. These are oracle lower bounds: an algorithm may make adaptive stochastic-gradient queries and randomize its output, but it does not have an exact function-value oracle. The proximal oracle in the nonsmooth construction is the same for both instances. The constructions use rare informative samples, as in standard lower-bound arguments.

\begin{proposition}[Clean smooth oracle]\label{prop:lower}
Fix $L,R,\sigma>0$, $1<p\le2$, $0<\delta\le1/8$, and an integer $T\ge b:=\log(1/(4\delta))$. For every algorithm making at most $T$ stochastic-gradient queries, there exists a one-dimensional instance satisfying the assumptions of Theorem~\ref{thm:main}, with $x_1=0$, such that its output $\widehat x$ satisfies
\[
\Prob\left\{f(\widehat x)-f(x^\star)\ge
\min\left\{\frac{LR^2}{2},\frac{\sigma R}{4}(b/T)^{(p-1)/p}\right\}\right\}\ge2\delta.
\]
\end{proposition}

\begin{proposition}[Clean nonsmooth composite oracle]\label{prop:nonsmoothlower}
Fix $G,R,\sigma>0$, $1<p\le2$, $0<\delta\le1/8$, and an integer $T\ge b:=\log(1/(4\delta))$. For every algorithm making at most $T$ stochastic-subgradient queries, even with exact proximal access to $r$, there exists a one-dimensional instance satisfying the assumptions of Theorem~\ref{thm:nonsmooth}, with $x_1=0$, such that
\[
\Prob\left\{F(\widehat x)-F^\star\ge
\min\left\{GR,\frac{\sigma R}{2}(b/T)^{(p-1)/p}\right\}\right\}\ge2\delta.
\]
\end{proposition}

\begin{proposition}[Corrupted smooth oracle]\label{prop:corruptlower}
Fix $L,R,\sigma>0$, $1<p\le2$, and integers $1\le q\le T$. For every algorithm making at most $T$ stochastic-gradient queries, there is a one-dimensional clean instance satisfying the assumptions of Theorem~\ref{thm:main}, with $x_1=0$, and a causal adversary satisfying \eqref{eq:corruptmodel} such that
\[
\Prob\left\{f(\widehat x)-f(x^\star)\ge
\frac R2\min\left\{LR,\frac\sigma2
\left(\frac q{2T}\right)^{(p-1)/p}\right\}\right\}\ge\frac14.
\]
\end{proposition}

\clearpage
\appendix
\section{Useful lemmas}\label{app:tools}
Throughout the appendix, denote $D(u,x)=\norm{u-x}^2/2$.

\subsection{Smoothness and proximal updates}
\begin{lemma}\label{lem:coco}
If $f$ is convex and $L$-smooth, then, for all $x,y\in\R^d$,
\begin{align}
f(x)&\ge f(y)+\ip{\nabla f(y)}{x-y}
+\frac{\norm{\nabla f(x)-\nabla f(y)}^2}{2L},\label{eq:strongsmooth}\\
\ip{\nabla f(x)-\nabla f(y)}{x-y}
&\ge\frac{\norm{\nabla f(x)-\nabla f(y)}^2}{L}.\label{eq:coco}
\end{align}
\end{lemma}
\begin{proof}
Fix $y$ and let $\psi(z)=f(z)-\ip{\nabla f(y)}{z}$. This function
is convex and $L$-smooth, with $\nabla\psi(y)=0$. Therefore
\[
\psi(z)\ge\psi(y)+\ip{\nabla\psi(y)}{z-y}=\psi(y),\qquad z\in\R^d.
\]
Applying this inequality and then smoothness at $x$ gives
\begin{align*}
\psi(y)
&\le\psi\left(x-\frac1L\nabla\psi(x)\right)\\
&\le\psi(x)+\left\langle\nabla\psi(x),-\frac1L\nabla\psi(x)\right\rangle
+\frac L2\left\|\frac1L\nabla\psi(x)\right\|^2\\
&=\psi(x)-\frac1{2L}\norm{\nabla\psi(x)}^2.
\end{align*}
Substituting the definition of $\psi$ yields
\begin{align*}
f(y)-\ip{\nabla f(y)}{y}
&\le f(x)-\ip{\nabla f(y)}{x}
-\frac1{2L}\norm{\nabla f(x)-\nabla f(y)}^2,\\
f(x)-f(y)
&\ge\ip{\nabla f(y)}{x-y}
+\frac1{2L}\norm{\nabla f(x)-\nabla f(y)}^2.
\end{align*}
This proves \eqref{eq:strongsmooth}. Interchanging $x$ and $y$ and
adding the resulting inequalities gives
\begin{align*}
0
&\ge \ip{\nabla f(y)}{x-y}+\ip{\nabla f(x)}{y-x}
+\frac1L\norm{\nabla f(x)-\nabla f(y)}^2\\
&=-\ip{\nabla f(x)-\nabla f(y)}{x-y}
+\frac1L\norm{\nabla f(x)-\nabla f(y)}^2,
\end{align*}
which is \eqref{eq:coco}.
\end{proof}

\begin{lemma}\label{lem:proxstep}
Let $F=f+r$ satisfy Assumption~\ref{ass:convex}. For $x\in\dom r$,
$h\in\partial f(x)$, an arbitrary error $e$, and $\alpha>0$, define
$x^+=\prox_{\alpha r}(x-\alpha(h+e))$.
If $f$ is $L$-smooth and $\alpha\le1/(2L)$, then, for every $u\in\dom r$,
\begin{equation}\label{eq:prox-smooth}
\alpha(F(x^+)-F(u))
\le D(u,x)-D(u,x^+)+\alpha\ip e{u-x}+\alpha^2\norm e^2.
\end{equation}
If $f$ is $G$-Lipschitz, then for every $\alpha>0$,
\begin{equation}\label{eq:prox-nonsmooth}
\alpha(F(x^+)-F(u))
\le D(u,x)-D(u,x^+)+\alpha\ip e{u-x}
+\alpha^2(\norm e^2+4G^2).
\end{equation}
\end{lemma}
\begin{proof}
The optimality condition for the proximal update gives
\[
\frac{x-x^+}{\alpha}-h-e\in\partial r(x^+).
\]
By the subgradient inequality,
\begin{align*}
r(u)&\ge r(x^+)+\left\langle\frac{x-x^+}{\alpha}-h-e,u-x^+\right\rangle,\\
\alpha\bigl(r(x^+)-r(u)+\ip{h+e}{x^+-u}\bigr)
&\le\ip{x-x^+}{x^+-u}.
\end{align*}
Expanding the square gives
\[
\norm{x-u}^2=\norm{x-x^+}^2+\norm{x^+-u}^2
+2\ip{x-x^+}{x^+-u}.
\]
Consequently,
\begin{equation}\label{eq:prox-threepoint}
\alpha\bigl(r(x^+)-r(u)+\ip{h+e}{x^+-u}\bigr)
\le D(u,x)-D(u,x^+)-\frac12\norm{x^+-x}^2.
\end{equation}

For smooth $f$, we have $h=\nabla f(x)$. Smoothness and convexity,
respectively, give
\begin{align*}
f(x^+)-f(x)&\le\ip h{x^+-x}+\frac L2\norm{x^+-x}^2,\\
f(x)-f(u)&\le\ip h{x-u}.
\end{align*}
Adding these two inequalities, we obtain
\begin{align*}
f(x^+)-f(u)
&=f(x^+)-f(x)+f(x)-f(u)\\
&\le\ip h{x^+-x}+\ip h{x-u}+\frac L2\norm{x^+-x}^2\\
&=\ip h{x^+-u}+\frac L2\norm{x^+-x}^2.
\end{align*}
Combining this with \eqref{eq:prox-threepoint} yields
\begin{align*}
\alpha(F(x^+)-F(u))
&\le D(u,x)-D(u,x^+)-\frac{1-L\alpha}{2}\norm{x^+-x}^2
+\alpha\ip e{u-x^+}\\
&=D(u,x)-D(u,x^+)+\alpha\ip e{u-x}
-\frac{1-L\alpha}{2}\norm{x^+-x}^2-\alpha\ip e{x^+-x}\\
&\le D(u,x)-D(u,x^+)+\alpha\ip e{u-x}
-\frac14\norm{x^+-x}^2+\alpha\norm e\norm{x^+-x}\\
&=D(u,x)-D(u,x^+)+\alpha\ip e{u-x}+\alpha^2\norm e^2
-\left(\frac{\norm{x^+-x}}2-\alpha\norm e\right)^2\\
&\le D(u,x)-D(u,x^+)+\alpha\ip e{u-x}+\alpha^2\norm e^2.
\end{align*}
The second inequality uses $L\alpha\le1/2$ and Cauchy--Schwarz.

For $G$-Lipschitz $f$, first note that $\norm h\le G$. If
$h\ne0$, then for any $a>0$,
\[
a\norm h
\le f\left(x+a\frac h{\norm h}\right)-f(x)
\le Ga.
\]
Hence, we can ensure that
\begin{align*}
f(x^+)-f(u)
&=f(x^+)-f(x)+f(x)-f(u)\\
&\le G\norm{x^+-x}+\ip h{x-u}\\
&=\ip h{x^+-u}+\ip h{x-x^+}+G\norm{x^+-x}\\
&\le\ip h{x^+-u}+2G\norm{x^+-x}.
\end{align*}
Using \eqref{eq:prox-threepoint} once more gives
\begin{align*}
\alpha(F(x^+)-F(u))
&\le D(u,x)-D(u,x^+)+\alpha\ip e{u-x^+}
+2\alpha G\norm{x^+-x}-\frac12\norm{x^+-x}^2\\
&\le D(u,x)-D(u,x^+)+\alpha\ip e{u-x}
+\alpha(\norm e+2G)\norm{x^+-x}-\frac12\norm{x^+-x}^2\\
&=D(u,x)-D(u,x^+)+\alpha\ip e{u-x}
+\frac{\alpha^2}2(\norm e+2G)^2\\
&\quad-\frac12\bigl(\norm{x^+-x}-\alpha(\norm e+2G)\bigr)^2\\
&\le D(u,x)-D(u,x^+)+\alpha\ip e{u-x}
+\alpha^2(\norm e^2+4G^2),
\end{align*}
where $(\norm e+2G)^2\le2\norm e^2+8G^2$ proves the last step.
\end{proof}

\begin{lemma}\label{lem:radius}
Under Assumptions~\ref{ass:convex} and \ref{ass:smooth}, let
$x^+=\prox_{\alpha r}(x-\alpha(\nabla f(x)+e))$
and set $0<\alpha\le1/L$. Then
\begin{equation}\label{eq:radiusrecursion}
\norm{x^+-x^\star}^2
\le\norm{x-x^\star}^2-2\alpha\ip e{x-x^\star}+2\alpha^2\norm e^2.
\end{equation}
\end{lemma}
\begin{proof}
For $u=\prox_{\alpha r}(a)$ and $v=\prox_{\alpha r}(b)$, the proximal
optimality conditions and the subgradient inequalities imply
\begin{align*}
\frac{a-u}{\alpha}&\in\partial r(u),
&\frac{b-v}{\alpha}&\in\partial r(v),\\
r(v)-r(u)&\ge\frac1\alpha\ip{a-u}{v-u},
&r(u)-r(v)&\ge\frac1\alpha\ip{b-v}{u-v}.
\end{align*}
Adding the last two inequalities and rearranging gives
\begin{align*}
0&\le\ip{a-b-(u-v)}{u-v},\\
\norm{u-v}^2&\le\ip{a-b}{u-v}\le\norm{a-b}\norm{u-v},\\
\norm{\prox_{\alpha r}(a)-\prox_{\alpha r}(b)}&\le\norm{a-b}.
\end{align*}
The last line also holds if $u=v$. Since $f$ is differentiable on
$\R^d$, optimality of $x^\star$ gives
\[
-\nabla f(x^\star)\in\partial r(x^\star),\qquad
x^\star=\prox_{\alpha r}(x^\star-\alpha\nabla f(x^\star)).
\]
Applying the preceding nonexpansiveness inequality, we obtain
\begin{align*}
\norm{x^+-x^\star}^2
&\le\norm{x-x^\star-\alpha(\nabla f(x)-\nabla f(x^\star)+e)}^2\\
&=\norm{x-x^\star}^2
-2\alpha\ip{\nabla f(x)-\nabla f(x^\star)}{x-x^\star}
-2\alpha\ip e{x-x^\star}\\
&\quad+\alpha^2\norm{\nabla f(x)-\nabla f(x^\star)+e}^2\\
&\le\norm{x-x^\star}^2-2\alpha\ip e{x-x^\star}
-\frac{2\alpha}{L}\norm{\nabla f(x)-\nabla f(x^\star)}^2\\
&\quad+2\alpha^2\norm{\nabla f(x)-\nabla f(x^\star)}^2+2\alpha^2\norm e^2\\
&=\norm{x-x^\star}^2-2\alpha\ip e{x-x^\star}+2\alpha^2\norm e^2
+2\alpha\left(\alpha-\frac1L\right)\norm{\nabla f(x)-\nabla f(x^\star)}^2\\
&\le\norm{x-x^\star}^2-2\alpha\ip e{x-x^\star}+2\alpha^2\norm e^2.
\end{align*}
Here the second inequality uses \eqref{eq:coco} and
$\norm{a+b}^2\le2\norm a^2+2\norm b^2$, and the last uses $\alpha\le1/L$.
\end{proof}

\subsection{A comparison inequality for the last iterate}
For positive learning rates $\eta_1,\ldots,\eta_T$, define
\begin{equation}\label{eq:potential}
\begin{gathered}
S_t=\sum_{j=t}^T\eta_j,\quad S_{T+1}=0,\quad
v_t=\frac{\eta_T}{S_t}\ (1\le t\le T),\quad v_0=v_1,\\
z_0=x^\star,\qquad
z_t=\frac{v_{t-1}}{v_t}z_{t-1}
+\left(1-\frac{v_{t-1}}{v_t}\right)x_t,\\
P_0=v_0D(x^\star,x_1),\qquad
P_t=v_tD(z_t,x_{t+1})\quad(1\le t\le T).
\end{gathered}
\end{equation}

\begin{lemma}\label{lem:prefix}
Let $F$ satisfy Assumption~\ref{ass:convex}, let $x_t\in\dom r$ for
$1\le t\le T+1$, and write $\Delta_t=F(x_t)-F^\star$.
Suppose that, for every $u\in\dom r$
and $1\le t\le T$,
\begin{equation}\label{eq:abstractstep}
\eta_t(F(x_{t+1})-F(u))
\le D(u,x_t)-D(u,x_{t+1})+\eta_t\ip{e_t}{u-x_t}+\rho_t,
\end{equation}
where $\rho_t\ge0$. Then, for every $1\le s\le T$,
\begin{align}
P_s+\eta_sv_s\Delta_{s+1}
+S_{s+1}\sum_{j=2}^s(v_j-v_{j-1})\Delta_j
\le P_0+\sum_{t=1}^s v_t\rho_t
+\sum_{t=1}^s\ip{e_t}{A_t},\label{eq:prefix}
\end{align}
where $A_t=\eta_tv_{t-1}(z_{t-1}-x_t)$.
Every term on the left-hand side is nonnegative.
\end{lemma}
\begin{proof}
By equation \eqref{eq:potential}, we know that
\begin{align*}
z_t-x_t&=\frac{v_{t-1}}{v_t}(z_{t-1}-x_t),\\
v_tD(z_t,x_t)
&=\frac{v_{t-1}^2}{v_t}D(z_{t-1},x_t)
\le v_{t-1}D(z_{t-1},x_t)=P_{t-1},\\
\eta_tv_t\ip{e_t}{z_t-x_t}
&=\eta_tv_{t-1}\ip{e_t}{z_{t-1}-x_t}=\ip{e_t}{A_t}.
\end{align*}
Putting $u=z_t$ in \eqref{eq:abstractstep} and multiplying by $v_t$ gives
\begin{equation}\label{eq:prefix-onestep}
P_t+\eta_tv_t(F(x_{t+1})-F(z_t))
\le P_{t-1}+v_t\rho_t+\ip{e_t}{A_t}.
\end{equation}
To expand $F(z_t)$, first telescope the recursion for $z_t$:
\begin{align*}
v_tz_t-v_{t-1}z_{t-1}&=(v_t-v_{t-1})x_t,\\
v_tz_t-v_0x^\star&=\sum_{j=1}^t(v_j-v_{j-1})x_j,\\
v_0+\sum_{j=1}^t(v_j-v_{j-1})&=v_t.
\end{align*}
Convexity therefore gives
\begin{align*}
v_tF(z_t)
&\le v_0F^\star+\sum_{j=1}^t(v_j-v_{j-1})F(x_j),\\
v_t(F(z_t)-F^\star)
&\le\sum_{j=1}^t(v_j-v_{j-1})\Delta_j,\\
v_t(F(x_{t+1})-F(z_t))
&\ge v_t\Delta_{t+1}-\sum_{j=1}^t(v_j-v_{j-1})\Delta_j.
\end{align*}
Substitution in \eqref{eq:prefix-onestep} and summation over $t=1,\ldots,s$
yield
\begin{align}
P_s+\sum_{t=1}^s\eta_tv_t\Delta_{t+1}
-\sum_{t=1}^s\eta_t\sum_{j=1}^t(v_j-v_{j-1})\Delta_j
\le P_0+\sum_{t=1}^s v_t\rho_t+\sum_{t=1}^s\ip{e_t}{A_t}.
\label{eq:prefix-summed}
\end{align}
The two sums containing $\Delta_j$ have the following exact expansions:
\begin{align*}
\sum_{t=1}^s\eta_tv_t\Delta_{t+1}
&=\eta_sv_s\Delta_{s+1}+\sum_{j=2}^s\eta_{j-1}v_{j-1}\Delta_j,\\
\sum_{t=1}^s\eta_t\sum_{j=1}^t(v_j-v_{j-1})\Delta_j
&=\sum_{j=1}^s(v_j-v_{j-1})\Delta_j\sum_{t=j}^s\eta_t\\
&=\sum_{j=2}^s(v_j-v_{j-1})(S_j-S_{s+1})\Delta_j.
\end{align*}
The last equality uses $v_1=v_0$. For $2\le j\le s$,
\begin{align*}
S_{j-1}-S_j&=\eta_{j-1},\\
v_j-v_{j-1}
&=\eta_T\left(\frac1{S_j}-\frac1{S_{j-1}}\right)
=\frac{\eta_T\eta_{j-1}}{S_jS_{j-1}},\\
\eta_{j-1}v_{j-1}-(v_j-v_{j-1})(S_j-S_{s+1})
&=S_j(v_j-v_{j-1})-(S_j-S_{s+1})(v_j-v_{j-1})\\
&=S_{s+1}(v_j-v_{j-1}).
\end{align*}
Consequently,
\begin{align*}
\sum_{t=1}^s\eta_tv_t\Delta_{t+1}
-\sum_{t=1}^s\eta_t\sum_{j=1}^t(v_j-v_{j-1})\Delta_j
&=\eta_sv_s\Delta_{s+1}
+S_{s+1}\sum_{j=2}^s(v_j-v_{j-1})\Delta_j.
\end{align*}
This proves \eqref{eq:prefix}. Its left-hand side is nonnegative because
$P_s\ge0$, $\Delta_j\ge0$, $S_{s+1}\ge0$, and $v_j\ge v_{j-1}$.
\end{proof}

\begin{lemma}\label{lem:weights}
For $\eta_t=\eta(T-t+1)/T$, with $\eta>0$, the quantities in
\eqref{eq:potential} satisfy
\begin{equation}\label{eq:weights}
v_t=\frac2{k(k+1)},\quad
P_0\le\frac{R^2}{T(T+1)}\le\frac{R^2}{T^2},\quad
\eta_t^2v_t\le\frac{2\eta^2}{T^2},\quad
\eta_t\sqrt{v_{t-1}}\le\frac{\sqrt2\eta}{T},
\end{equation}
where $k=T-t+1$. Moreover, $v_T=1$ and $\eta_T=\eta/T$.
\end{lemma}
\begin{proof}
For $k=T-t+1$, the sum defining $S_t$ becomes
\begin{align*}
S_t&=\frac\eta T\sum_{j=t}^T(T-j+1)
=\frac\eta T\sum_{i=1}^k i=\frac{\eta k(k+1)}{2T},\\
v_t&=\frac{\eta/T}{\eta k(k+1)/(2T)}=\frac2{k(k+1)},\\
P_0&=\frac2{T(T+1)}\frac{\norm{x^\star-x_1}^2}{2}
\le\frac{R^2}{T(T+1)}\le\frac{R^2}{T^2},\\
\eta_t^2v_t&=\frac{\eta^2k^2}{T^2}\frac2{k(k+1)}
=\frac{2\eta^2}{T^2}\frac{k}{k+1}\le\frac{2\eta^2}{T^2}.
\end{align*}
Since $v_{t-1}\le v_t$, the last bound also gives
\[
\eta_t\sqrt{v_{t-1}}\le\eta_t\sqrt{v_t}
=\sqrt{\eta_t^2v_t}\le\frac{\sqrt2\eta}{T}.
\]
Putting $t=T$ in the definitions gives $v_T=1$ and $\eta_T=\eta/T$.
All these calculations include $T=1$.
\end{proof}

\subsection{Clipping error}
\begin{lemma}\label{lem:clip}
Let $\lambda>0$, $1<p\le2$, and let $h$ be measurable with respect to the conditioning sigma-field,
with $\norm h\le\lambda/2$. Suppose that $g=h+\varepsilon$,
$\Et\varepsilon=0$, and $\Et\norm\varepsilon^p\le\sigma^p$.
Define $e=\clip(g,\lambda)-h$, $b=\Et e$, and $\xi=e-b$. Then
\begin{align}
\norm e&\le\frac32\lambda,
&\Et\norm e^2&\le V:=2\sigma^p\lambda^{2-p},\label{eq:clipsecond}\\
\norm b&\le B:=2\sigma^p\lambda^{1-p},
&\norm\xi&\le2\lambda,
&\Et\norm\xi^2&\le V.\label{eq:clipbias}
\end{align}
\end{lemma}
\begin{proof}
Clipping is Euclidean projection onto the ball of radius $\lambda$.
Its nonexpansiveness follows from the proximal argument in
Lemma~\ref{lem:radius}, with $r$ the indicator of that ball. Since
$\clip(h,\lambda)=h$,
\begin{align*}
\norm e
&=\norm{\clip(g,\lambda)-\clip(h,\lambda)}\le\norm{g-h}=\norm\varepsilon,\\
\norm e
&\le\norm{\clip(g,\lambda)}+\norm h\le\lambda+\lambda/2=3\lambda/2,\\
\Et\norm e^2
&\le(3\lambda/2)^{2-p}\Et\norm e^p
\le(3\lambda/2)^{2-p}\Et\norm\varepsilon^p
\le2\sigma^p\lambda^{2-p}.
\end{align*}
The last line uses $0\le2-p<1$; when $p=2$, its first inequality is
simply an equality.

For the bias, $\Et(g-h)=0$ implies
\[
b=\Et[\clip(g,\lambda)-g],\qquad
\norm{\clip(g,\lambda)-g}=(\norm g-\lambda)\ind_{\{\norm g>\lambda\}}.
\]
On $\{\norm g>\lambda\}$, the triangle inequality gives
\[
\norm\varepsilon\ge\norm g-\norm h>\lambda/2,
\qquad \norm g-\lambda\le\norm\varepsilon+\norm h-\lambda\le\norm\varepsilon.
\]
Thus
\begin{align*}
\norm b
&\le\Et\norm{\clip(g,\lambda)-g}\\
&\le\Et\bigl[\norm\varepsilon\ind_{\{\norm\varepsilon>\lambda/2\}}\bigr]\\
&=\Et\bigl[\norm\varepsilon^p\norm\varepsilon^{1-p}
\ind_{\{\norm\varepsilon>\lambda/2\}}\bigr]\\
&\le(\lambda/2)^{1-p}\Et\norm\varepsilon^p
\le2\sigma^p\lambda^{1-p}.
\end{align*}
The product in the third line is evaluated only on
$\{\norm\varepsilon>\lambda/2\}$. Finally,
\begin{align*}
\xi&=\clip(g,\lambda)-\Et\clip(g,\lambda),
&\Et\xi&=0,\\
\norm\xi&\le\norm{\clip(g,\lambda)}+\Et\norm{\clip(g,\lambda)}\le2\lambda,\\
\Et\norm\xi^2
&=\Et\bigl[\norm e^2-2\ip e b+\norm b^2\bigr]\\
&=\Et\norm e^2-2\ip{\Et e}b+\norm b^2
=\Et\norm e^2-\norm b^2\le V.
\end{align*}
Here $b$ is measurable before the current sample is observed.
\end{proof}

\subsection{Concentration inequalities}
The following lemmas use the filtration $\F_0,\ldots,\F_T$ and
$\Et[\cdot]=\E[\cdot\mid\F_{t-1}]$.

\begin{lemma}\label{lem:nonnegative}
Suppose $X_t$ is $\F_t$-measurable, $0\le X_t\le H$ almost surely,
and $\Et X_t\le m_t$, where $m_t\ge0$ is $\F_{t-1}$-measurable and
$\sum_{t=1}^T m_t\le A$ almost surely for a deterministic $A$.
For $u>0$,
\[
\Prob\left\{\sum_{t=1}^T X_t\le2A+Hu\right\}\ge1-e^{-u}.
\]
\end{lemma}
\begin{proof}
If $H=0$, then $X_t=0$ almost surely. For $H>0$, convexity of the
exponential gives, for $0\le x\le H$,
\[
\exp(x/H)\le(1-x/H)\exp(0)+(x/H)\exp(1)
=1+(\mathrm e-1)x/H.
\]
Consequently,
\[
\Et\exp(X_t/H)\le1+(\mathrm e-1)\Et X_t/H
\le1+(\mathrm e-1)m_t/H\le\exp((\mathrm e-1)m_t/H).
\]
Set
\[
Z_s=\exp\left(\frac1H\sum_{t=1}^s X_t
-\frac{\mathrm e-1}{H}\sum_{t=1}^s m_t\right),\qquad Z_0=1.
\]
Since $Z_{s-1}$ and $m_s$ are $\F_{s-1}$-measurable,
\begin{align*}
\E[Z_s\mid\F_{s-1}]
&=Z_{s-1}\exp\left(-\frac{\mathrm e-1}{H}m_s\right)
\E[\exp(X_s/H)\mid\F_{s-1}]\le Z_{s-1},\\
\E Z_T&\le\E Z_{T-1}\le\cdots\le\E Z_0=1.
\end{align*}
Using $\sum m_t\le A$, $\mathrm e-1<2$, and Markov's inequality,
\begin{align*}
\Prob\left\{\sum_{t=1}^T X_t>2A+Hu\right\}
&\le\Prob\left\{\sum_{t=1}^T X_t>(\mathrm e-1)A+Hu\right\}\\
&\le\Prob\{Z_T>e^u\}\le e^{-u}\E Z_T\le e^{-u}.
\end{align*}
\end{proof}

\begin{lemma}\label{lem:freedman}
Let $M_s=\sum_{t=1}^s Y_t$, where $Y_t$ is $\F_t$-measurable,
$\Et Y_t=0$, $|Y_t|\le a$, and $\sum_{t=1}^T\Et Y_t^2\le w$ almost
surely for deterministic $a,w\ge0$. For every $u>0$,
\begin{equation}\label{eq:freedman}
\Prob\left\{\max_{1\le s\le T}M_s\le\sqrt{2wu}+2au\right\}\ge1-e^{-u}.
\end{equation}
\end{lemma}
\begin{proof}
If $a=0$, all increments vanish. If $w=0$, then
$\sum_t\E Y_t^2=\E\sum_t\Et Y_t^2=0$, so the same conclusion holds.
Assume $a,w>0$ and fix $0<\theta<3/a$. For $k\ge2$,
$k!\ge2\cdot3^{k-2}$ and $|Y_t|^k\le a^{k-2}Y_t^2$. Hence
\begin{align*}
\Et e^{\theta Y_t}
&=1+\theta\Et Y_t+\sum_{k=2}^\infty\frac{\theta^k}{k!}\Et Y_t^k\\
&\le1+\sum_{k=2}^\infty\frac{\theta^ka^{k-2}}{k!}\Et Y_t^2\\
&\le1+\frac{\theta^2}2\Et Y_t^2\sum_{k=2}^\infty(\theta a/3)^{k-2}\\
&=1+\frac{\theta^2\Et Y_t^2}{2(1-\theta a/3)}
\le\exp\left(\frac{\theta^2\Et Y_t^2}{2(1-\theta a/3)}\right).
\end{align*}
Boundedness of $Y_t$ justifies taking conditional expectations of
the absolutely convergent series. Define
\[
W_s=\sum_{t=1}^s\Et Y_t^2,\qquad
Z_s=\exp\left(\theta M_s-\frac{\theta^2W_s}{2(1-\theta a/3)}\right),
\qquad Z_0=1.
\]
The preceding conditional bound gives
\[
\E[Z_s\mid\F_{s-1}]
=Z_{s-1}\exp\left(-\frac{\theta^2\E[Y_s^2\mid\F_{s-1}]}{2(1-\theta a/3)}\right)
\E[e^{\theta Y_s}\mid\F_{s-1}]
\le Z_{s-1}.
\]
For $c>1$, let $\tau$ be the first $s\in\{1,\ldots,T\}$ with $Z_s\ge c$,
and let $\tau=T+1$ if there is no such index. Since
$\{\tau\ge s\}\in\F_{s-1}$,
\begin{align*}
Z_{s\wedge\tau}-Z_{(s-1)\wedge\tau}
&=\ind_{\{\tau\ge s\}}(Z_s-Z_{s-1}),\\
\E[Z_{s\wedge\tau}-Z_{(s-1)\wedge\tau}\mid\F_{s-1}]
&\le0,\\
c\Prob\{\tau\le T\}&\le\E Z_{T\wedge\tau}\le1.
\end{align*}
Thus $\Prob\{\max_{s\le T}Z_s\ge c\}\le c^{-1}$.

For a threshold $r>0$, take $\theta=r/(w+ar/3)$. Then $\theta a<3$ and
\[
1-\theta a/3=\frac{w}{w+ar/3},\qquad
\theta r-\frac{\theta^2w}{2(1-\theta a/3)}
=\frac{r^2}{2(w+ar/3)}.
\]
Since $W_s\le w$, we obtain
\begin{align*}
\Prob\{\max_{s\le T}M_s\ge r\}
&\le\Prob\left\{\max_{s\le T}Z_s
\ge\exp\left(\frac{r^2}{2(w+ar/3)}\right)\right\}\\
&\le\exp\left(-\frac{r^2}{2(w+ar/3)}\right).
\end{align*}
For $r=\sqrt{2wu}+2au$, direct expansion gives
\begin{align*}
r^2&=2wu+4au\sqrt{2wu}+4a^2u^2,\\
2u(w+ar/3)&=2wu+\frac23au\sqrt{2wu}+\frac43a^2u^2,\\
r^2-2u(w+ar/3)&=\frac{10}3au\sqrt{2wu}+\frac83a^2u^2\ge0.
\end{align*}
The last probability is therefore at most $e^{-u}$. This proves the
stated maximal form of Freedman's inequality \citep{freedman1975};
the maximum over $s$ requires no additional union bound.
\end{proof}

\section{Proof of Theorem~\ref{thm:main}}\label{app:main}
Write $K=128$ and $\Delta_t=f(x_t)-f(x^\star)$. The parameters are
\[
\ell=\log(8/\delta),\qquad
\lambda=\max\{4LR,\sigma(T/\ell)^{1/p}\},\qquad
\eta=\frac{R}{K\lambda\ell},\qquad
\eta_t=\eta\frac{T-t+1}{T}.
\]

\subsection{One-step inequalities and the stopping rule}
Define
\begin{equation}\label{eq:main-error}
e_t=\clip(g_t,\lambda)-\nabla f(x_t),\qquad
b_t=\Et e_t,\qquad \xi_t=e_t-b_t.
\end{equation}
Then $\Et\xi_t=0$ and
$x_{t+1}=x_t-\eta_t(\nabla f(x_t)+e_t)$. Moreover,
\[
\eta_t\le\eta=\frac{R}{K\lambda\ell}
\le\frac{R}{4KLR\ell}=\frac1{4KL\ell}\le\frac1{2L}.
\]
Lemmas~\ref{lem:proxstep} and \ref{lem:radius}, with $r=0$, imply
\begin{align}
\eta_t(f(x_{t+1})-f(u))
&\le D(u,x_t)-D(u,x_{t+1})
+\eta_t\ip{e_t}{u-x_t}+\eta_t^2\norm{e_t}^2,\label{eq:main-step}\\
\norm{x_{t+1}-x^\star}^2
&\le\norm{x_t-x^\star}^2
-2\eta_t\ip{e_t}{x_t-x^\star}+2\eta_t^2\norm{e_t}^2.
\label{eq:main-radius}
\end{align}
Use $S_t,v_t,z_t,P_t$ from \eqref{eq:potential} and
$A_t=\eta_tv_{t-1}(z_{t-1}-x_t)$. Applying Lemma~\ref{lem:prefix}
to \eqref{eq:main-step} gives, for every $1\le s\le T$,
\begin{align*}
P_s+\eta_sv_s\Delta_{s+1}
&\le P_s+\eta_sv_s\Delta_{s+1}
+S_{s+1}\sum_{j=2}^s(v_j-v_{j-1})\Delta_j\\
&\le P_0+\sum_{t=1}^s\eta_t^2v_t\norm{e_t}^2
+\sum_{t=1}^s\ip{e_t}{A_t}.
\end{align*}
In particular, the following slightly looser bound will also be used
in the smooth extensions:
\begin{equation}\label{eq:main-prefix}
P_s+\eta_sv_s\Delta_{s+1}
\le P_0+2\sum_{t=1}^s\eta_t^2v_t\norm{e_t}^2
+\sum_{t=1}^s\ip{e_t}{A_t}.
\end{equation}
Lemma~\ref{lem:weights} gives
\begin{equation}\label{eq:main-weights}
P_0\le\frac{R^2}{T^2},\qquad
2\eta_t^2v_t\le\frac{4\eta^2}{T^2},\qquad
\eta_t\sqrt{v_{t-1}}\le\frac{\sqrt2\eta}{T}.
\end{equation}

To bound the gradient and the coefficient $A_t$, let
\[
\tau=\min\left\{t\in\{1,\ldots,T+1\}:
\norm{x_t-x^\star}>2R\ \text{or}\ P_{t-1}>\frac{4R^2}{T^2}\right\},
\qquad I_t=\ind_{\{t<\tau\}},
\]
with $\min\varnothing=\infty$. For $1\le t\le T$,
\[
\{t<\tau\}
=\bigcap_{j=1}^t\left\{\norm{x_j-x^\star}\le2R,
\ P_{j-1}\le\frac{4R^2}{T^2}\right\}\in\F_{t-1}.
\]
Indeed, each $x_j$ and $P_{j-1}$ is known before query $j$.
Thus $I_t$ is predictable, meaning that it is known before $g_t$ is
observed. The initial bounds imply $\tau\ne1$.

On $\{I_t=1\}$, unconstrained optimality gives $\nabla f(x^\star)=0$, so
\[
\norm{\nabla f(x_t)}
=\norm{\nabla f(x_t)-\nabla f(x^\star)}
\le L\norm{x_t-x^\star}\le2LR\le\lambda/2.
\]
Lemma~\ref{lem:clip} therefore gives, on this predictable event,
\begin{align}
\norm{e_t}&\le\tfrac32\lambda,
&\norm{b_t}&\le B:=2\sigma^p\lambda^{1-p},
&\Et\norm{e_t}^2&\le V:=2\sigma^p\lambda^{2-p},\label{eq:main-moments}\\
\norm{\xi_t}&\le3\lambda,
&\Et\norm{\xi_t}^2&\le V.\notag
\end{align}
We use $3\lambda$ in place of the lemma's bound $2\lambda$.
Since $\lambda^p\ge\sigma^pT/\ell$,
\begin{equation}\label{eq:main-budgets}
BT=2\lambda\frac{\sigma^pT}{\lambda^p}\le2\lambda\ell,
\qquad
VT=2\lambda^2\frac{\sigma^pT}{\lambda^p}\le2\lambda^2\ell.
\end{equation}
Also, $P_{t-1}=v_{t-1}\norm{z_{t-1}-x_t}^2/2$ implies
\begin{align}
I_t\norm{A_t}
&=I_t\eta_tv_{t-1}\norm{z_{t-1}-x_t}
=I_t\eta_t\sqrt{v_{t-1}}\sqrt{2P_{t-1}}\notag\\
&\le\frac{\sqrt2\eta}{T}\sqrt{\frac{8R^2}{T^2}}
=\frac{4\eta R}{T^2}.\label{eq:main-A}
\end{align}

\subsection{Bounds for the stopped errors}
First, by predictability and \eqref{eq:main-moments},
\begin{align*}
0\le I_t\norm{e_t}^2&\le\frac94\lambda^2\le3\lambda^2,\\
\Et[I_t\norm{e_t}^2]&=I_t\Et\norm{e_t}^2\le I_tV,\\
\sum_{t=1}^T\Et[I_t\norm{e_t}^2]&\le VT.
\end{align*}
Lemma~\ref{lem:nonnegative}, with $H=3\lambda^2$, $A=VT$, and $u=\ell$,
gives, with probability at least $1-e^{-\ell}$,
\begin{equation}\label{eq:main-Q}
\sum_{t=1}^T I_t\norm{e_t}^2
\le2VT+3\lambda^2\ell
\le4\lambda^2\ell+3\lambda^2\ell=7\lambda^2\ell.
\end{equation}

For the radius recursion, define
\[
M_s^r=-2\sum_{t=1}^s I_t\eta_t\ip{\xi_t}{x_t-x^\star},\qquad M_0^r=0.
\]
Its conditional means, increment bound, and variance bound are
\begin{align*}
\Et[M_t^r-M_{t-1}^r]
&=-2I_t\eta_t\ip{\Et\xi_t}{x_t-x^\star}=0,\\
|M_t^r-M_{t-1}^r|
&\le2I_t\eta_t\norm{\xi_t}\norm{x_t-x^\star}
\le2\eta(3\lambda)(2R)=12\eta R\lambda,\\
\Et[(M_t^r-M_{t-1}^r)^2]
&\le4I_t\eta_t^2\norm{x_t-x^\star}^2\Et\norm{\xi_t}^2
\le16\eta^2R^2V,\\
\sum_{t=1}^T\Et[(M_t^r-M_{t-1}^r)^2]
&\le16\eta^2R^2VT.
\end{align*}
Lemma~\ref{lem:freedman}, with $a=12\eta R\lambda$,
$w=16\eta^2R^2VT$, and $u=\ell$, gives, with probability at least
$1-e^{-\ell}$,
\begin{align}
\max_{s\le T}M_s^r
&\le\sqrt{2(16\eta^2R^2VT)\ell}+2(12\eta R\lambda)\ell\notag\\
&=\eta R\sqrt{32VT\ell}+24\eta R\lambda\ell\notag\\
&\le\eta R\sqrt{64\lambda^2\ell^2}+24\eta R\lambda\ell
=32\eta R\lambda\ell=\frac{32R^2}{K}.\label{eq:main-Mr}
\end{align}
The corresponding bias is bounded on every sample path:
\begin{align}
2\sum_{t=1}^T I_t\eta_t|\ip{b_t}{x_t-x^\star}|
&\le2\sum_{t=1}^T I_t\eta_t\norm{b_t}\norm{x_t-x^\star}\notag\\
&\le4\eta RBT\le8\eta R\lambda\ell=\frac{8R^2}{K}.
\label{eq:main-rbias}
\end{align}

For the comparison potential, define
\[
M_s^P=\sum_{t=1}^s I_t\ip{\xi_t}{A_t},\qquad M_0^P=0.
\]
The vector $A_t$ is $\F_{t-1}$-measurable. By \eqref{eq:main-A},
\begin{align*}
\Et[M_t^P-M_{t-1}^P]&=I_t\ip{\Et\xi_t}{A_t}=0,\\
|M_t^P-M_{t-1}^P|
&\le I_t\norm{\xi_t}\norm{A_t}
\le(3\lambda)\frac{4\eta R}{T^2}=\frac{12\eta R\lambda}{T^2},\\
\Et[(M_t^P-M_{t-1}^P)^2]
&\le I_t\norm{A_t}^2\Et\norm{\xi_t}^2
\le\frac{16\eta^2R^2V}{T^4},\\
\sum_{t=1}^T\Et[(M_t^P-M_{t-1}^P)^2]
&\le\frac{16\eta^2R^2VT}{T^4}.
\end{align*}
Applying Lemma~\ref{lem:freedman} with these two bounds gives, with
probability at least $1-e^{-\ell}$,
\begin{align}
\max_{s\le T}M_s^P
&\le\sqrt{2\frac{16\eta^2R^2VT}{T^4}\ell}
+2\frac{12\eta R\lambda}{T^2}\ell\notag\\
&=\frac{\eta R}{T^2}\sqrt{32VT\ell}
+\frac{24\eta R\lambda\ell}{T^2}\notag\\
&\le\frac{8\eta R\lambda\ell}{T^2}
+\frac{24\eta R\lambda\ell}{T^2}
=\frac{32R^2}{KT^2}.\label{eq:main-MP}
\end{align}
The bias satisfies
\begin{align}
\sum_{t=1}^T I_t|\ip{b_t}{A_t}|
&\le\sum_{t=1}^T I_t\norm{b_t}\norm{A_t}
\le\frac{4\eta RBT}{T^2}\notag\\
&\le\frac{8\eta R\lambda\ell}{T^2}
=\frac{8R^2}{KT^2}.\label{eq:main-Pbias}
\end{align}
Let $E$ be the event on which \eqref{eq:main-Q}, \eqref{eq:main-Mr},
and \eqref{eq:main-MP} all hold. A union bound gives
\[
\Prob(E^c)\le3e^{-\ell}=3\delta/8\le\delta.
\]
Both martingale bounds hold for every prefix on $E$.

\subsection{Removing the stopping rule and bounding the last iterate}
Fix a sample path in $E$ and an index $1\le s\le T$ with $s<\tau$.
Then $I_t=1$ for all $t\le s$. Summing \eqref{eq:main-radius} gives
\begin{align*}
\norm{x_{s+1}-x^\star}^2-\norm{x_1-x^\star}^2
&=\sum_{t=1}^s\bigl(\norm{x_{t+1}-x^\star}^2-\norm{x_t-x^\star}^2\bigr)\\
&\le-2\sum_{t=1}^s\eta_t\ip{b_t+\xi_t}{x_t-x^\star}
+2\sum_{t=1}^s\eta_t^2\norm{e_t}^2.
\end{align*}
Using \eqref{eq:main-Q}, \eqref{eq:main-Mr}, and \eqref{eq:main-rbias},
we obtain
\begin{align}
\norm{x_{s+1}-x^\star}^2
&\le R^2+M_s^r+2\sum_{t=1}^s I_t\eta_t|\ip{b_t}{x_t-x^\star}|
+2\eta^2\sum_{t=1}^T I_t\norm{e_t}^2\notag\\
&\le R^2+\frac{32R^2}{K}+\frac{8R^2}{K}+14\eta^2\lambda^2\ell\notag\\
&=\left(1+\frac{40}{K}+\frac{14}{K^2\ell}\right)R^2
<2R^2.\label{eq:main-close-radius}
\end{align}
Using \eqref{eq:main-prefix} and the bounds for the comparison potential,
\begin{align}
P_s+\eta_sv_s\Delta_{s+1}
&\le P_0+2\sum_{t=1}^s\eta_t^2v_t\norm{e_t}^2
+\sum_{t=1}^s\ip{b_t+\xi_t}{A_t}\notag\\
&\le\frac{R^2}{T^2}+\frac{4\eta^2}{T^2}\sum_{t=1}^T I_t\norm{e_t}^2
+M_s^P+\sum_{t=1}^s I_t|\ip{b_t}{A_t}|\notag\\
&\le\frac{R^2}{T^2}+\frac{28\eta^2\lambda^2\ell}{T^2}
+\frac{32R^2}{KT^2}+\frac{8R^2}{KT^2}\notag\\
&=\left(1+\frac{40}{K}+\frac{28}{K^2\ell}\right)\frac{R^2}{T^2}
<\frac{2R^2}{T^2}.\label{eq:main-close-potential}
\end{align}
The parameter substitutions and numerical inequalities used above are
\[
\eta^2\lambda^2\ell=\frac{R^2}{K^2\ell},\qquad
1+\frac{40}{K}+\frac{28}{K^2\ell}
\le1+\frac{40}{128}+\frac{28}{128^2}<2,
\]
because $\ell>1$.

If $\tau\le T+1$, then $\tau\ne1$ and $s=\tau-1\in\{1,\ldots,T\}$.
Since $s<\tau$, the two preceding estimates apply and imply
\[
\norm{x_\tau-x^\star}^2<2R^2<4R^2,\qquad
P_{\tau-1}\le P_{\tau-1}+\eta_{\tau-1}v_{\tau-1}\Delta_\tau
<\frac{2R^2}{T^2}<\frac{4R^2}{T^2}.
\]
Neither stopping condition holds at $\tau$, a contradiction. Thus
$\tau=\infty$ on $E$, and $\norm{x_t-x^\star}\le2R$ for $1\le t\le T+1$.
Taking $s=T$ in \eqref{eq:main-close-potential} now gives
\begin{align*}
\frac\eta T\bigl(f(x_{T+1})-f(x^\star)\bigr)
&=\eta_Tv_T\Delta_{T+1}
\le P_T+\eta_Tv_T\Delta_{T+1}\le\frac{2R^2}{T^2},\\
f(x_{T+1})-f(x^\star)
&\le\frac{2R^2}{\eta T}
=\frac{2KR\lambda\ell}{T}=\frac{256R\lambda\ell}{T}\\
&\le\frac{256R\ell}{T}\left(4LR+\sigma(T/\ell)^{1/p}\right)\\
&=1024\frac{LR^2\ell}{T}
+256\sigma R\left(\frac\ell T\right)^{(p-1)/p}.
\end{align*}
The last inequality uses $\max\{a,b\}\le a+b$ for $a,b\ge0$.
\hfill$\square$

\section{Proof of Theorem~\ref{thm:nonsmooth}}\label{app:nonsmooth}

Let $K=128$. The parameters in the theorem are
\[
\ell=\log(4/\delta),\qquad
\lambda=\max\{2G,\sigma(T/\ell)^{1/p}\},\qquad
H=G\sqrt T+\lambda\ell,\qquad
\eta=\frac{R}{KH},\qquad
\eta_t=\eta\frac{T-t+1}{T}.
\]
Since $G+\sigma>0$, we have $\lambda,H>0$.
Write $\Delta_t=F(x_t)-F^\star$ and
$D(u,x)=\norm{u-x}^2/2$. The proximal updates belong to $\dom r$,
and $\Delta_t\ge0$ because $x^\star$ minimizes $F$.

\subsection{The one-step inequality and its telescoping sum}

Let $h_t=\Et g_t\in\partial f(x_t)$ and define
\begin{equation}\label{eq:ns-error}
e_t=\clip(g_t,\lambda)-h_t=b_t+\xi_t,
\qquad b_t=\Et e_t,\qquad \Et\xi_t=0.
\end{equation}
Lipschitz continuity gives $\norm{h_t}\le G\le\lambda/2$.
Lemma~\ref{lem:clip} therefore yields, for every $t$,
\begin{align}
\norm{b_t}&\le B:=2\sigma^p\lambda^{1-p},
&\Et\norm{e_t}^2&\le V:=2\sigma^p\lambda^{2-p},\label{eq:ns-moments}\\
\norm{e_t}&\le\frac32\lambda,
&\norm{\xi_t}&\le3\lambda.
\label{eq:ns-bounded}
\end{align}
The last bound is a relaxation of the lemma's $2\lambda$ bound. Also,
\begin{align*}
\Et\norm{\xi_t}^2
&=\Et\norm{e_t-b_t}^2
=\Et\norm{e_t}^2-2\ip{\Et e_t}{b_t}+\norm{b_t}^2\\
&=\Et\norm{e_t}^2-\norm{b_t}^2\le V.
\end{align*}
Since $\lambda^p\ge\sigma^pT/\ell$,
\begin{equation}\label{eq:ns-budgets}
BT=2\sigma^pT\lambda^{1-p}\le2\lambda\ell,
\qquad
VT=2\sigma^pT\lambda^{2-p}\le2\lambda^2\ell.
\end{equation}

Since $h_t+e_t=\clip(g_t,\lambda)$, the update is
$x_{t+1}=\prox_{\eta_t r}(x_t-\eta_t(h_t+e_t))$.
Apply the nonsmooth inequality of Lemma~\ref{lem:proxstep} with
$x=x_t$, $x^+=x_{t+1}$, $h=h_t$, $e=e_t$, and $\alpha=\eta_t$.
For every $u\in\dom r$, this gives
\begin{equation}\label{eq:ns-onestep}
\eta_t\bigl(F(x_{t+1})-F(u)\bigr)
\le D(u,x_t)-D(u,x_{t+1})
  +\eta_t\ip{e_t}{u-x_t}
  +\eta_t^2(\norm{e_t}^2+4G^2).
\end{equation}

Use $S_t,v_t,z_t,P_t$ from~\eqref{eq:potential} and set
$A_t=\eta_tv_{t-1}(z_{t-1}-x_t)$. Applying
Lemma~\ref{lem:prefix} to~\eqref{eq:ns-onestep} with
$\rho_t=\eta_t^2(\norm{e_t}^2+4G^2)$ gives, for every $1\le s\le T$,
\begin{align}
P_s+\eta_sv_s\Delta_{s+1}
 +S_{s+1}\sum_{j=2}^s(v_j-v_{j-1})\Delta_j
\le{}&P_0+\sum_{t=1}^s\eta_t^2v_t(\norm{e_t}^2+4G^2)
\notag\\
&+\sum_{t=1}^s\ip{e_t}{A_t}.
\label{eq:ns-prefix}
\end{align}
All terms on the left are nonnegative.

The linear learning-rate schedule gives
\begin{align*}
S_t&=\frac\eta T\sum_{j=1}^{T-t+1}j
=\frac{\eta(T-t+1)(T-t+2)}{2T},\\
v_t&=\frac2{(T-t+1)(T-t+2)},\\
\eta_t^2v_t
&=\frac{2\eta^2}{T^2}\frac{T-t+1}{T-t+2}
\le\frac{2\eta^2}{T^2}.
\end{align*}
In particular,
\begin{equation}\label{eq:ns-weights}
\begin{aligned}
P_0\le\frac{R^2}{T(T+1)}\le\frac{R^2}{T^2},\qquad
&\eta_t^2v_t\le\frac{2\eta^2}{T^2},\\
\eta_t\sqrt{v_{t-1}}\le\eta_t\sqrt{v_t}
&\le\frac{\sqrt2\eta}{T}.
\end{aligned}
\end{equation}
These inequalities also hold at $t=1$, where $v_0=v_1$.

\subsection{Bounds for the stopped error sums}

Define
\[
\tau=\min\left\{t\in\{1,\ldots,T+1\}:
  P_{t-1}>\frac{4R^2}{T^2}\right\},\qquad
I_t=\ind_{\{t<\tau\}},\quad 1\le t\le T,
\]
with $\min\varnothing=\infty$. The initial bound in
\eqref{eq:ns-weights} gives $\tau\ne1$. Since $x_t,z_{t-1},P_{t-1}$
are $\F_{t-1}$-measurable and
\[
\{t<\tau\}
=\bigcap_{j=1}^t\left\{P_{j-1}\le\frac{4R^2}{T^2}\right\}
\in\F_{t-1},
\]
both $I_t$ and $A_t$ are measurable before $g_t$ is sampled.
When $I_t=1$,
\begin{align}
\norm{A_t}
&=\eta_tv_{t-1}\norm{z_{t-1}-x_t}
=\eta_t\sqrt{v_{t-1}}\sqrt{2P_{t-1}}
\notag\\
&\le\frac{\sqrt2\eta}{T}\sqrt{\frac{8R^2}{T^2}}
=\frac{4\eta R}{T^2}.
\label{eq:ns-coefficient}
\end{align}

First, \eqref{eq:ns-moments}--\eqref{eq:ns-bounded} imply
$0\le\norm{e_t}^2\le(9/4)\lambda^2\le3\lambda^2$ and
$\Et\norm{e_t}^2\le V$. Lemma~\ref{lem:nonnegative}, with
upper bound $3\lambda^2$, mean budget $VT$, and $u=\ell$, gives
an event $\mathcal E_Q$ of probability at least $1-e^{-\ell}$ on which
\begin{equation}\label{eq:ns-Q}
\sum_{t=1}^T\norm{e_t}^2
\le2VT+3\lambda^2\ell
\le4\lambda^2\ell+3\lambda^2\ell
=7\lambda^2\ell.
\end{equation}
No stopping is needed for this sum because the moment bounds hold
at every iterate.

For the bias terms, \eqref{eq:ns-coefficient},
\eqref{eq:ns-budgets}, and $H\ge\lambda\ell$ give
\begin{align}
\sum_{t=1}^T I_t\left|\ip{b_t}{A_t}\right|
&\le\sum_{t=1}^T I_t\norm{b_t}\norm{A_t}
\le\frac{4\eta RBT}{T^2}
\notag\\
&\le\frac{8\eta R\lambda\ell}{T^2}
=\frac{8R^2}{KT^2}\frac{\lambda\ell}{H}
\le\frac{8R^2}{KT^2}.
\label{eq:ns-bias}
\end{align}

For the centered terms, let
\[
M_s=\sum_{t=1}^s I_t\ip{\xi_t}{A_t},\qquad M_0=0.
\]
Predictability and~\eqref{eq:ns-error} give
\[
\Et(M_t-M_{t-1})
=I_t\ip{\Et\xi_t}{A_t}=0.
\]
Thus $(M_s)_{s=0}^T$ is a martingale. Its increments and their
conditional second moments satisfy
\begin{align}
|M_t-M_{t-1}|
&\le I_t\norm{\xi_t}\norm{A_t}
\le3\lambda\frac{4\eta R}{T^2}
=\frac{12\eta R\lambda}{T^2},
\label{eq:ns-increment}\\
\sum_{t=1}^T\Et(M_t-M_{t-1})^2
&\le\sum_{t=1}^T I_t\norm{A_t}^2\Et\norm{\xi_t}^2
\notag\\
&\le\sum_{t=1}^T\frac{16\eta^2R^2V}{T^4}
=\frac{16\eta^2R^2VT}{T^4}.
\label{eq:ns-variance}
\end{align}
Apply Lemma~\ref{lem:freedman} with the right-hand sides of
\eqref{eq:ns-increment} and~\eqref{eq:ns-variance} as $a$ and $w$,
respectively, and with $u=\ell$. On an event $\mathcal E_M$ of
probability at least $1-e^{-\ell}$,
\begin{align}
\max_{1\le s\le T}M_s
&\le\sqrt{\frac{32\eta^2R^2VT\ell}{T^4}}
  +\frac{24\eta R\lambda\ell}{T^2}
\notag\\
&=\frac{\eta R}{T^2}\sqrt{32VT\ell}
  +\frac{24\eta R\lambda\ell}{T^2}
\notag\\
&\le\frac{\eta R}{T^2}\sqrt{64\lambda^2\ell^2}
  +\frac{24\eta R\lambda\ell}{T^2}
\notag\\
&=\frac{32\eta R\lambda\ell}{T^2}
=\frac{32R^2}{KT^2}\frac{\lambda\ell}{H}
\le\frac{32R^2}{KT^2}.
\label{eq:ns-martingale}
\end{align}
The third line uses $VT\le2\lambda^2\ell$. This single event
controls all prefixes $s$.

\subsection{Closing the stopping argument}

Let $\mathcal E=\mathcal E_Q\cap\mathcal E_M$. Then
\[
\Prob(\mathcal E)\ge1-2e^{-\ell}
=1-\delta/2\ge1-\delta.
\]
Fix a sample path in $\mathcal E$ and an index $1\le s\le T$ with
$s<\tau$. All $I_t$ with $t\le s$ equal one, so
\begin{align*}
\sum_{t=1}^s\ip{e_t}{A_t}
&=\sum_{t=1}^s I_t\ip{b_t}{A_t}+M_s\\
&\le\sum_{t=1}^T I_t\left|\ip{b_t}{A_t}\right|
  +\max_{1\le j\le T}M_j\\
&\le\frac{8R^2}{KT^2}+\frac{32R^2}{KT^2}
=\frac{40R^2}{KT^2}.
\end{align*}
The other error terms in~\eqref{eq:ns-prefix} satisfy
\begin{align}
\sum_{t=1}^s\eta_t^2v_t(\norm{e_t}^2+4G^2)
&\le\frac{2\eta^2}{T^2}
  \left(\sum_{t=1}^T\norm{e_t}^2+4G^2T\right)
\notag\\
&\le\frac{2\eta^2}{T^2}(7\lambda^2\ell+4G^2T)
\notag\\
&=\frac{14\eta^2\lambda^2\ell+8\eta^2G^2T}{T^2}.
\label{eq:ns-totalcost}
\end{align}
Here the first inequality uses~\eqref{eq:ns-weights} and
nonnegativity, and the second uses~\eqref{eq:ns-Q}. Substituting
$\eta=R/(KH)$ gives
\[
\eta^2\lambda^2\ell
=\frac{R^2}{K^2\ell}\left(\frac{\lambda\ell}{H}\right)^2
\le\frac{R^2}{K^2\ell},\qquad
\eta^2G^2T
=\frac{R^2}{K^2}\left(\frac{G\sqrt T}{H}\right)^2
\le\frac{R^2}{K^2}.
\]
Drop the nonnegative sum of gaps from~\eqref{eq:ns-prefix} and
combine the preceding bounds:
\begin{align}
P_s+\eta_sv_s\Delta_{s+1}
&\le\frac{R^2}{T^2}
 +\frac{14\eta^2\lambda^2\ell+8\eta^2G^2T}{T^2}
 +\frac{40R^2}{KT^2}
\notag\\
&\le\left(1+\frac{40}{K}+\frac{14}{K^2\ell}
                 +\frac8{K^2}\right)\frac{R^2}{T^2}
\notag\\
&\le\left(1+\frac{40}{128}+\frac{22}{128^2}\right)
 \frac{R^2}{T^2}
<\frac{2R^2}{T^2}.
\label{eq:ns-closed}
\end{align}
The last line uses $\ell>\log4>1$ and $K=128$.

Suppose $\tau\le T+1$. Since $\tau\ne1$, the index
$s=\tau-1$ lies in $\{1,\ldots,T\}$ and satisfies $s<\tau$.
Thus~\eqref{eq:ns-closed} applies at this index and gives
\[
P_{\tau-1}=P_s
\le P_s+\eta_sv_s\Delta_{s+1}
<\frac{2R^2}{T^2}<\frac{4R^2}{T^2},
\]
contrary to the definition of $\tau$. Hence $\tau=\infty$ on
$\mathcal E$, and~\eqref{eq:ns-closed} holds for every $s\le T$.
At $s=T$, $v_T=1$ and $\eta_T=\eta/T$, so
\[
\frac\eta T\Delta_{T+1}
\le P_T+\frac\eta T\Delta_{T+1}
<\frac{2R^2}{T^2}.
\]
It follows that
\begin{align*}
F(x_{T+1})-F^\star
&\le\frac{2R^2}{\eta T}
=\frac{2KRH}{T}
=\frac{256RH}{T}\\
&=256\frac{GR}{\sqrt T}+256\frac{R\lambda\ell}{T}\\
&\le256\frac{GR}{\sqrt T}
 +256\frac{R\ell}{T}\left(2G+\sigma(T/\ell)^{1/p}\right)\\
&=256\frac{GR}{\sqrt T}
 +512\frac{GR\ell}{T}
 +256\sigma R(\ell/T)^{(p-1)/p}.
\end{align*}
We used $\max\{2G,\sigma(T/\ell)^{1/p}\}
\le2G+\sigma(T/\ell)^{1/p}$. If $T\ge\ell^2$, then
\[
256\frac{GR}{\sqrt T}+512\frac{GR\ell}{T}
=\frac{GR}{\sqrt T}\left(256+512\frac\ell{\sqrt T}\right)
\le768\frac{GR}{\sqrt T}.
\]
Both claimed bounds therefore hold on $\mathcal E$.\hfill$\square$

\section{Proof of Theorem~\ref{thm:composite}}\label{app:composite}

\subsection{Construction of the anchor}

Let $g^{(0)}_1,\ldots,g^{(0)}_m$ be the pilot gradients at $x_1$, and let
$\mathcal H_j$ contain the information available after the first $j$
pilot queries. In particular, the conditional oracle assumption gives
\[
\E[g^{(0)}_j\mid\mathcal H_{j-1}]=\nabla f(x_1),\qquad
\E[\norm{g^{(0)}_j-\nabla f(x_1)}^p\mid\mathcal H_{j-1}]\le\sigma^p.
\]

\begin{lemma}\label{lem:composite-anchor}
Let $m=\lceil32\log(2/\delta)\rceil$, and let
$c\in\arg\min_{u\in\R^d}\sum_{j=1}^m\norm{g^{(0)}_j-u}$ be a measurable
choice of geometric median. Then
\[
\Prob\{\norm{c-\nabla f(x_1)}\le24\sigma\}\ge1-\delta/2.
\]
\end{lemma}
\begin{proof}
Write $u_0=\nabla f(x_1)$. If $\sigma=0$, the conditional moment
bound implies $g^{(0)}_j=u_0$ almost surely for every $j$. The geometric
median is then $u_0$. Suppose that $\sigma>0$, and define
\[
J_j=\ind_{\{\norm{g^{(0)}_j-u_0}>8\sigma\}}.
\]
Conditional Markov's inequality gives
\[
\E[J_j\mid\mathcal H_{j-1}]\le
\frac{\E[\norm{g^{(0)}_j-u_0}^p\mid\mathcal H_{j-1}]}{(8\sigma)^p}
\le8^{-p}\le\frac18.
\]
Since $J_j\in\{0,1\}$, we have $2^{J_j}=1+J_j$. Thus
\begin{align*}
\E[2^{J_j}\mid\mathcal H_{j-1}]
&=1+\E[J_j\mid\mathcal H_{j-1}]\le\frac98,\\
\E\left[2^{\sum_{j=1}^mJ_j}\right]
&=\E\left[2^{\sum_{j=1}^{m-1}J_j}
\E[2^{J_m}\mid\mathcal H_{m-1}]\right]\\
&\le\frac98\E\left[2^{\sum_{j=1}^{m-1}J_j}\right]
\le\left(\frac98\right)^m.
\end{align*}
Markov's inequality now yields
\begin{align*}
\Prob\left\{\sum_{j=1}^mJ_j>\frac m4\right\}
&\le2^{-m/4}\left(\frac98\right)^m\\
&=\exp\left[-m\left(\frac{\log2}{4}-\log\frac98\right)\right]
\le e^{-m/32}\le\frac\delta2,
\end{align*}
where $(\log2)/4-\log(9/8)>1/32$. The pilot samples may
depend on the preceding pilot history.

On the complementary event, let
$\mathcal G=\{j:\norm{g^{(0)}_j-u_0}\le8\sigma\}$. Then
$|\mathcal G|\ge3m/4$. The triangle inequality gives
\[
\norm{g^{(0)}_j-c}-\norm{g^{(0)}_j-u_0}
\ge
\begin{cases}
\norm{c-u_0}-16\sigma,&j\in\mathcal G,\\
-\norm{c-u_0},&j\notin\mathcal G.
\end{cases}
\]
By the minimizing property of $c$,
\begin{align*}
0
&\ge\sum_{j=1}^m\bigl(\norm{g^{(0)}_j-c}-\norm{g^{(0)}_j-u_0}\bigr)\\
&\ge(2|\mathcal G|-m)\norm{c-u_0}-16\sigma|\mathcal G|,\\
\norm{c-u_0}
&\le\frac{16\sigma|\mathcal G|}{2|\mathcal G|-m}
\le24\sigma.
\end{align*}
The denominator is positive, and the last inequality is equivalent to
$|\mathcal G|\ge3m/4$. This proves the claim.
\end{proof}

\subsection{Proximal recursions and localization}

Let $\mathcal A=\{\norm{c-\nabla f(x_1)}\le24\sigma\}$. Fix a
pilot history in $\mathcal A$, outside a set of probability zero, and
work under the conditional law given $\mathcal H_m$. The anchor $c$
is now fixed. The main-stage filtration $\{\F_t\}$ includes
$\mathcal H_m$, and $\Et$ denotes conditional expectation given
$\F_{t-1}$. The oracle assumptions hold with respect to this filtration.
All concentration probabilities below refer to this conditional law.

Write $\Delta_t=F(x_t)-F^\star$ and define
\[
\eta=\frac{R}{128\lambda\ell},\qquad
e_t=c+\clip(g_t-c,\lambda)-\nabla f(x_t),\qquad
b_t=\Et e_t,\qquad \xi_t=e_t-b_t,
\]
so that $\Et\xi_t=0$. Set
$B=2\sigma^p\lambda^{1-p}$ and $V=2\sigma^p\lambda^{2-p}$.
The parameter choice gives $\lambda\ge6LR$ and
$\lambda^p\ge\sigma^pT/\ell$. Hence
\begin{equation}\label{eq:comp-parameters}
\begin{aligned}
\eta_t\le\eta
&=\frac{R}{128\lambda\ell}
\le\frac{1}{768L\ell}\le\frac1{2L},\\
BT
&=2\sigma^pT\lambda^{1-p}
\le2\lambda^p\ell\lambda^{1-p}=2\lambda\ell,\\
VT
&=2\sigma^pT\lambda^{2-p}
\le2\lambda^p\ell\lambda^{2-p}=2\lambda^2\ell.
\end{aligned}
\end{equation}

The update is $x_{t+1}=\prox_{\eta_t r}(x_t-\eta_t(\nabla f(x_t)+e_t))$.
Apply Lemma~\ref{lem:proxstep} with $x=x_t$, $h=\nabla f(x_t)$,
$e=e_t$, and $\alpha=\eta_t\le1/(2L)$. For every $u\in\dom r$,
\begin{equation}\label{eq:comp-onestep}
\eta_t(F(x_{t+1})-F(u))
\le D(u,x_t)-D(u,x_{t+1})
+\eta_t\ip{e_t}{u-x_t}+\eta_t^2\norm{e_t}^2.
\end{equation}
With $x=x_t$, $e=e_t$, and $\alpha=\eta_t\le1/L$,
Lemma~\ref{lem:radius} also gives
\[
\norm{x_{t+1}-x^\star}^2
\le\norm{x_t-x^\star}^2
-2\eta_t\ip{e_t}{x_t-x^\star}+2\eta_t^2\norm{e_t}^2.
\]
The radius lemma applies to composite objectives and does not require
$\nabla f(x^\star)=0$.

To control the final objective value, define
\[
\begin{gathered}
S_t=\sum_{j=t}^T\eta_j,\quad S_{T+1}=0,\quad
v_t=\frac{\eta_T}{S_t},\quad v_0=v_1,\quad z_0=x^\star,\\
z_t=\frac{v_{t-1}}{v_t}z_{t-1}
+\left(1-\frac{v_{t-1}}{v_t}\right)x_t,\quad
P_t=v_tD(z_t,x_{t+1}),\quad P_0=v_0D(x^\star,x_1),\\
A_t=\eta_tv_{t-1}(z_{t-1}-x_t).
\end{gathered}
\]
The weights $v_t$ increase, so $z_t\in\dom r$.
Applying Lemma~\ref{lem:prefix} to \eqref{eq:comp-onestep} with
$\rho_t=\eta_t^2\norm{e_t}^2$, we obtain, for every $1\le s\le T$,
\begin{align*}
P_s+\eta_sv_s\Delta_{s+1}
+S_{s+1}\sum_{j=2}^s(v_j-v_{j-1})\Delta_j
\le P_0+\sum_{t=1}^s\eta_t^2v_t\norm{e_t}^2
+\sum_{t=1}^s\ip{e_t}{A_t}.
\end{align*}
Since $S_{s+1}\ge0$, $v_j-v_{j-1}\ge0$, and $\Delta_j\ge0$,
we can drop the sum on the left. Enlarging the nonnegative
squared-error term gives
\begin{equation}\label{eq:comp-prefix}
P_s+\eta_sv_s\Delta_{s+1}
\le P_0+2\sum_{t=1}^s\eta_t^2v_t\norm{e_t}^2
+\sum_{t=1}^s\ip{e_t}{A_t}.
\end{equation}
For $k=T-t+1$, the linear schedule gives
$S_t=(\eta/T)\sum_{j=1}^k j=\eta k(k+1)/(2T)$. Hence
\begin{equation}\label{eq:comp-weights}
\begin{aligned}
v_t&=\frac2{k(k+1)},\qquad
P_0\le\frac{R^2}{T(T+1)}\le\frac{R^2}{T^2},\\
2\eta_t^2v_t
&=\frac{4\eta^2}{T^2}\frac{k}{k+1}\le\frac{4\eta^2}{T^2},
\qquad
\eta_t\sqrt{v_{t-1}}\le\frac{\sqrt2\eta}{T}.
\end{aligned}
\end{equation}
For the last inequality, when $t\ge2$ its left-hand side is
$\sqrt2\eta k/(T\sqrt{(k+1)(k+2)})$; when $t=1$, it is
$\eta\sqrt{2/(T(T+1))}$. Both expressions have the stated bound,
including when $T=1$. Also, $v_T=1$ and $\eta_T=\eta/T$.

Let $\tau$ be the first $t\in\{1,\ldots,T+1\}$ for which
$\norm{x_t-x^\star}>2R$ or $P_{t-1}>4R^2/T^2$; set
$\tau=\infty$ if neither condition occurs. Initially,
$\norm{x_1-x^\star}\le R$ and $P_0\le R^2/T^2$, so $\tau>1$.
Since $x_t$ and $P_{t-1}$ are known before query $t$,
$I_t=\ind_{\{t<\tau\}}$ is $\F_{t-1}$-measurable. On $\{t<\tau\}$,
\begin{align*}
\norm{\nabla f(x_t)-c}
&\le\norm{\nabla f(x_t)-\nabla f(x_1)}
+\norm{\nabla f(x_1)-c}\\
&\le L\norm{x_t-x_1}+24\sigma\\
&\le L\bigl(\norm{x_t-x^\star}+\norm{x_1-x^\star}\bigr)+24\sigma
\le3LR+24\sigma\le\lambda/2.
\end{align*}
Conditionally on $\F_{t-1}$, the shifted sample $g_t-c$ has mean
$\nabla f(x_t)-c$ and centered noise $g_t-\nabla f(x_t)$. We can
therefore apply Lemma~\ref{lem:clip} on the predictable event
$\{t<\tau\}$ to obtain
\[
\begin{gathered}
I_t\norm{e_t}\le\tfrac32\lambda,\qquad
I_t\norm{b_t}\le B,\qquad
\Et[I_t\norm{e_t}^2]\le I_tV,\\
I_t\norm{\xi_t}\le2\lambda\le3\lambda,\qquad
\Et[I_t\norm{\xi_t}^2]\le I_tV.
\end{gathered}
\]
Moreover, $A_t$ is $\F_{t-1}$-measurable, and
\begin{equation}\label{eq:comp-A}
\begin{aligned}
I_t\norm{A_t}
&=I_t\eta_t\sqrt{v_{t-1}}\sqrt{2P_{t-1}}\\
&\le\frac{\sqrt2\eta}{T}\sqrt{\frac{8R^2}{T^2}}
=\frac{4\eta R}{T^2}.
\end{aligned}
\end{equation}
All these bounds apply to stopped variables; no bound on the full
trajectory has yet been assumed.

\subsection{Concentration and completion of the proof}

The variables $I_t\norm{e_t}^2$ lie in $[0,3\lambda^2]$, and their
conditional means sum to at most $VT$. Lemma~\ref{lem:nonnegative}
with exponent $\ell$ gives, with failure probability at most $e^{-\ell}$,
\begin{equation}\label{eq:comp-Q}
Q:=\sum_{t=1}^T I_t\norm{e_t}^2
\le2VT+3\lambda^2\ell
\le4\lambda^2\ell+3\lambda^2\ell=7\lambda^2\ell.
\end{equation}
Every prefix of this sum is at most $Q$ because its terms are nonnegative.

For the distance recursion, define
$M_s^r=-2\sum_{t=1}^sI_t\eta_t\ip{\xi_t}{x_t-x^\star}$,
with $M_0^r=0$. Predictability and $\Et\xi_t=0$ show that
$\Et(M_t^r-M_{t-1}^r)=0$. Its increments and conditional second
moments satisfy
\begin{align*}
|M_t^r-M_{t-1}^r|
&\le2I_t\eta_t\norm{\xi_t}\norm{x_t-x^\star}
\le12\eta R\lambda=\frac{12R^2}{128\ell},\\
\sum_{t=1}^T\Et[(M_t^r-M_{t-1}^r)^2]
&\le4\sum_{t=1}^TI_t\eta_t^2\norm{x_t-x^\star}^2\Et\norm{\xi_t}^2\\
&\le16\eta^2R^2VT
\le32\eta^2R^2\lambda^2\ell
=\frac{32R^4}{128^2\ell}.
\end{align*}
Lemma~\ref{lem:freedman} consequently yields, with failure probability
at most $e^{-\ell}$,
\begin{equation}\label{eq:comp-Mr}
\begin{aligned}
\max_{1\le s\le T}M_s^r
&\le\sqrt{2\ell\frac{32R^4}{128^2\ell}}
+2\ell\frac{12R^2}{128\ell}\\
&=\frac{8R^2}{128}+\frac{24R^2}{128}=\frac{R^2}{4}.
\end{aligned}
\end{equation}
The bias term in that recursion has the deterministic bound
\begin{equation}\label{eq:comp-rbias}
\begin{aligned}
2\sum_{t=1}^TI_t\eta_t|\ip{b_t}{x_t-x^\star}|
&\le4\eta RBT\le8\eta R\lambda\ell\\
&=\frac{8R^2}{128}=\frac{R^2}{16}.
\end{aligned}
\end{equation}

For \eqref{eq:comp-prefix}, define
$M_s^P=\sum_{t=1}^sI_t\ip{\xi_t}{A_t}$, with $M_0^P=0$.
Since $I_tA_t$ is predictable, $\Et(M_t^P-M_{t-1}^P)=0$.
By \eqref{eq:comp-A},
\begin{align*}
|M_t^P-M_{t-1}^P|
&\le I_t\norm{\xi_t}\norm{A_t}
\le\frac{12\eta R\lambda}{T^2}
=\frac{12R^2}{128\ell T^2},\\
\sum_{t=1}^T\Et[(M_t^P-M_{t-1}^P)^2]
&\le\sum_{t=1}^TI_t\norm{A_t}^2\Et\norm{\xi_t}^2\\
&\le\frac{16\eta^2R^2VT}{T^4}
\le\frac{32\eta^2R^2\lambda^2\ell}{T^4}
=\frac{32R^4}{128^2\ell T^4}.
\end{align*}
A second application of Lemma~\ref{lem:freedman} gives, with failure
probability at most $e^{-\ell}$,
\begin{equation}\label{eq:comp-MP}
\begin{aligned}
\max_{1\le s\le T}M_s^P
&\le\sqrt{2\ell\frac{32R^4}{128^2\ell T^4}}
+2\ell\frac{12R^2}{128\ell T^2}\\
&=\frac{8R^2}{128T^2}+\frac{24R^2}{128T^2}
=\frac{R^2}{4T^2}.
\end{aligned}
\end{equation}
The bias contribution to \eqref{eq:comp-prefix} satisfies
\begin{equation}\label{eq:comp-Pbias}
\begin{aligned}
\sum_{t=1}^TI_t|\ip{b_t}{A_t}|
&\le\frac{4\eta RBT}{T^2}
\le\frac{8\eta R\lambda\ell}{T^2}\\
&=\frac{8R^2}{128T^2}=\frac{R^2}{16T^2}.
\end{aligned}
\end{equation}

Let $\mathcal E$ be the intersection of the events in
\eqref{eq:comp-Q}, \eqref{eq:comp-Mr}, and \eqref{eq:comp-MP}.
For every fixed pilot history under consideration,
$\Prob(\mathcal E^c\mid\mathcal H_m)\le3e^{-\ell}$.
On $\mathcal E$, fix $1\le s\le T$ with $s<\tau$.
Then $I_t=1$ for $1\le t\le s$. Sum the distance recursion and
write $e_t=\xi_t+b_t$ to obtain
\begin{align}
\norm{x_{s+1}-x^\star}^2
&\le R^2+M_s^r
-2\sum_{t=1}^sI_t\eta_t\ip{b_t}{x_t-x^\star}
+2\sum_{t=1}^sI_t\eta_t^2\norm{e_t}^2\notag\\
&\le R^2+\frac{R^2}{4}+\frac{R^2}{16}+2\eta^2Q\notag\\
&\le\left(\frac{21}{16}+\frac{14}{128^2\ell}\right)R^2
<2R^2.
\label{eq:comp-closed-radius}
\end{align}
For the potential, \eqref{eq:comp-prefix} and
\eqref{eq:comp-weights} give
\begin{align}
P_s+\eta_sv_s\Delta_{s+1}
&\le P_0+M_s^P+\sum_{t=1}^sI_t\ip{b_t}{A_t}
+2\sum_{t=1}^sI_t\eta_t^2v_t\norm{e_t}^2\notag\\
&\le\frac{R^2}{T^2}+\frac{R^2}{4T^2}+\frac{R^2}{16T^2}
+\frac{4\eta^2Q}{T^2}\notag\\
&\le\left(\frac{21}{16}+\frac{28}{128^2\ell}\right)\frac{R^2}{T^2}
<\frac{2R^2}{T^2}.
\label{eq:comp-closed-potential}
\end{align}
Both strict inequalities follow from $\ell>\log16>1$ and
$21/16+28/128^2<2$.

If $\tau\le T+1$, then $s=\tau-1$ lies in $\{1,\ldots,T\}$
and satisfies $s<\tau$. The preceding two inequalities would give
\[
\norm{x_\tau-x^\star}^2<2R^2<4R^2,\qquad
P_{\tau-1}\le P_{\tau-1}+\eta_{\tau-1}v_{\tau-1}\Delta_\tau
<\frac{2R^2}{T^2}<\frac{4R^2}{T^2},
\]
where $\Delta_\tau\ge0$ because $x^\star$ minimizes $F$.
Neither stopping condition could hold at $\tau$, a contradiction.
Thus $\tau=\infty$ on $\mathcal E$, and the claimed distance bound
holds for every $1\le t\le T+1$.

Finally, take $s=T$ in \eqref{eq:comp-closed-potential}. Since
$P_T\ge0$, $v_T=1$, and $\eta_T=\eta/T$,
\[
\frac\eta T\bigl(F(x_{T+1})-F^\star\bigr)
\le P_T+\eta_Tv_T\Delta_{T+1}
\le\frac{2R^2}{T^2},\qquad
F(x_{T+1})-F^\star
\le\frac{2R^2}{\eta T}=\frac{256R\lambda\ell}{T}.
\]
To remove the conditioning on the pilot, the failure probability of
the stated conclusions is at most
\begin{align*}
\Prob(\mathcal A^c)
+\E\bigl[\ind_{\mathcal A}\Prob(\mathcal E^c\mid\mathcal H_m)\bigr]
&\le\frac\delta2+3e^{-\ell}
=\frac\delta2+\frac{3\delta}{16}
=\frac{11\delta}{16}<\delta.
\end{align*}

For $T\ge\ell$, the bound $\max\{a,b\}\le a+b$ for
nonnegative $a,b$ and the definition of $\lambda$ yield
\begin{align*}
\frac{256R\lambda\ell}{T}
&\le\frac{256R\ell}{T}
\left(6LR+48\sigma+\sigma(T/\ell)^{1/p}\right)\\
&=1536\frac{LR^2\ell}{T}
+12288\frac{\sigma R\ell}{T}
+256\sigma R(\ell/T)^{(p-1)/p}\\
&\le1536\frac{LR^2\ell}{T}
+12544\sigma R(\ell/T)^{(p-1)/p}.
\end{align*}
The last step uses $0<\ell/T\le1$ and
$\ell/T\le(\ell/T)^{(p-1)/p}$.
There are $m$ pilot calls and $T$ main-stage calls, so $N=T+m$.
If $N\ge2m$ and $T=N-m$, then $T\ge N/2$ and $T\ge m\ge\ell$.
In particular, the preceding bound implies
\[
F(x_{T+1})-F^\star
\le3072\frac{LR^2\ell}{N}
+12544\cdot2^{(p-1)/p}\sigma R(\ell/N)^{(p-1)/p},
\]
which has the stated order in the total number of oracle calls.
\hfill$\square$

\section{Proof of Theorem~\ref{thm:corrupt}}\label{app:corrupt}
\subsection{The clean error and the replacement error}
Recall that $\ell=\log(8/\delta)$ and $\Gamma=\ell+q$. Write $K=128$,
and retain the notation
\[
 \eta=\frac{R}{K\lambda\Gamma},\qquad
 \eta_t=\eta\frac{T-t+1}{T},\qquad
 B=2\sigma^p\lambda^{1-p},\qquad
 V=2\sigma^p\lambda^{2-p}.
\]
The choice $\lambda=\max\{4LR,\sigma(T/\Gamma)^{1/p}\}$ gives
\begin{align}
 \sigma^pT&\le\lambda^p\Gamma,\notag\\
 BT=2\sigma^pT\lambda^{1-p}
 &\le2\lambda^p\Gamma\lambda^{1-p}=2\lambda\Gamma,\notag\\
 VT=2\sigma^pT\lambda^{2-p}
 &\le2\lambda^p\Gamma\lambda^{2-p}=2\lambda^2\Gamma,\notag\\
 \eta_t\le\eta=\frac{R}{K\lambda\Gamma}
 &\le\frac{R}{K(4LR)\Gamma}
 =\frac{1}{4KL\Gamma}\le\frac{1}{2L}.
 \label{eq:corrupt-parameters}
\end{align}
The last inequality uses $K=128$ and $\Gamma\ge\ell>\log8$.

Use $S_t,v_t,z_t,P_t$ from \eqref{eq:potential}, with
$f^\star=f(x^\star)$ and $\Delta_t=f(x_t)-f^\star$. Define
\[
 \tau=\min\left\{t\in\{1,\ldots,T+1\}:
       \norm{x_t-x^\star}>2R
       \ \text{or}\ 
       P_{t-1}>\frac{4R^2}{T^2}\right\},
 \qquad I_t=\ind_{\{t<\tau\}},
\]
where $\min\varnothing=\infty$. The filtration contains the algorithm's
randomness and all preceding latent clean samples and adversarial choices.
Thus $x_t,P_{t-1},I_t$ are $\F_{t-1}$-measurable. Also,
\[
 \norm{x_1-x^\star}\le R<2R,\qquad
 P_0\le\frac{R^2}{T(T+1)}<\frac{4R^2}{T^2},
 \qquad\text{so }\tau>1.
\]

Put $h_t=\nabla f(x_t)$. We separate the error of the latent clean
sample from the change made by the adversary:
\begin{equation}\label{eq:corrupt-decomposition}
 \begin{aligned}
 u_t&=\clip(g_t,\lambda)-h_t,&
 d_t&=\clip(\widetilde g_t,\lambda)-\clip(g_t,\lambda),\\
 b_t&=\Et u_t,&
 \xi_t&=u_t-b_t,\\
 e_t&=\clip(\widetilde g_t,\lambda)-h_t
      =u_t+d_t=b_t+\xi_t+d_t.&
 \end{aligned}
\end{equation}
Here $\Et\xi_t=\Et u_t-b_t=0$. We center the latent clean sample
before deciding whether it is replaced. In particular, no unbiasedness
assumption is made after conditioning on the event $C_t=0$.

On $\{I_t=1\}$, smoothness and $\nabla f(x^\star)=0$ imply
\[
 \norm{h_t}
 =\norm{\nabla f(x_t)-\nabla f(x^\star)}
 \le L\norm{x_t-x^\star}\le2LR\le\lambda/2.
\]
Lemma~\ref{lem:clip}, applied conditionally on $\F_{t-1}$, therefore gives
\begin{equation}\label{eq:corrupt-clean-moments}
 \begin{gathered}
 \norm{u_t}\le3\lambda/2,\qquad
 \norm{b_t}\le B,\qquad
 \Et\norm{u_t}^2\le V,\\
 \norm{\xi_t}\le2\lambda,\qquad
 \Et\norm{\xi_t}^2\le V
 \qquad\text{on }\{I_t=1\}.
 \end{gathered}
\end{equation}
On the same event, both clipped vectors have norm at most $\lambda$, so
\begin{align*}
 \norm{e_t}
 &\le\norm{\clip(\widetilde g_t,\lambda)}+\norm{h_t}
 \le3\lambda/2,\\
 \norm{d_t}
 &\le\bigl(\norm{\clip(\widetilde g_t,\lambda)}
             +\norm{\clip(g_t,\lambda)}\bigr)C_t
 \le2\lambda C_t.
\end{align*}
The second bound also holds without stopping. If $C_t=0$, then $e_t=u_t$;
if $C_t=1$ and $I_t=1$, then $\norm{e_t}^2\le9\lambda^2/4$.
Consequently,
\begin{equation}\label{eq:corrupt-pathwise}
 \norm{d_t}\le2\lambda C_t,\qquad
 I_t\norm{e_t}^2
 \le I_t\norm{u_t}^2+3\lambda^2 I_tC_t.
\end{equation}
These inequalities hold on every sample path, although $C_t$ may depend
on the current clean sample.

\subsection{Bounds for the three random sums}
Apply Lemma~\ref{lem:nonnegative} to the sequence
$I_t\norm{u_t}^2$, with $H=3\lambda^2$, $m_t=I_tV$, and $u=\ell$.
Its assumptions follow from
\[
 0\le I_t\norm{u_t}^2\le3\lambda^2,\qquad
 \Et[I_t\norm{u_t}^2]=I_t\Et\norm{u_t}^2\le I_tV,\qquad
 \sum_{t=1}^T I_tV\le VT.
\]
Thus, on an event $E_Q$ with $\Prob(E_Q)\ge1-e^{-\ell}$,
\begin{align}
 Q_e:=\sum_{t=1}^T I_t\norm{e_t}^2
 &\le\sum_{t=1}^T I_t\norm{u_t}^2
       +3\lambda^2\sum_{t=1}^T I_tC_t\notag\\
 &\le2VT+3\lambda^2\ell+3\lambda^2q\notag\\
 &\le4\lambda^2\Gamma+3\lambda^2(\ell+q)
 =7\lambda^2\Gamma.
 \label{eq:corrupt-Q}
\end{align}
Every summand in $Q_e$ is nonnegative, so this also bounds each prefix.
The bias and the replacements have the pathwise bound
\begin{align}
 \sum_{t=1}^T I_t(\norm{b_t}+\norm{d_t})
 &\le BT+2\lambda\sum_{t=1}^T I_tC_t\notag\\
 &\le2\lambda\Gamma+2\lambda q
 \le4\lambda\Gamma.
 \label{eq:corrupt-nonmartingale-budget}
\end{align}

For the distance to $x^\star$, define the martingale
\[
 M_s^r=-2\sum_{t=1}^s I_t\eta_t\ip{\xi_t}{x_t-x^\star},
 \qquad M_0^r=0.
\]
Predictability and $\Et\xi_t=0$ give
\[
 \Et(M_t^r-M_{t-1}^r)
 =-2I_t\eta_t\ip{\Et\xi_t}{x_t-x^\star}=0.
\]
Using $\norm{\xi_t}\le2\lambda\le3\lambda$ on $\{I_t=1\}$, we obtain
\begin{align*}
 |M_t^r-M_{t-1}^r|
 &\le2I_t\eta_t\norm{\xi_t}\norm{x_t-x^\star}
 \le2\eta(3\lambda)(2R)=12\eta R\lambda,\\
 \Et(M_t^r-M_{t-1}^r)^2
 &=4I_t\eta_t^2\Et\bigl[\ip{\xi_t}{x_t-x^\star}^2\bigr]\\
 &\le4I_t\eta_t^2\norm{x_t-x^\star}^2\Et\norm{\xi_t}^2
 \le16\eta^2R^2V,\\
 \sum_{t=1}^T\Et(M_t^r-M_{t-1}^r)^2
 &\le16\eta^2R^2VT
 \le32\eta^2R^2\lambda^2\Gamma.
\end{align*}
Lemma~\ref{lem:freedman} with $u=\ell$ gives an event $E_r$,
of probability at least $1-e^{-\ell}$, on which
\begin{align}
 \max_{s\le T}M_s^r
 &\le\sqrt{32\eta^2R^2VT\ell}+24\eta R\lambda\ell\notag\\
 &\le\sqrt{64\eta^2R^2\lambda^2\Gamma\ell}
       +24\eta R\lambda\ell\notag\\
 &=8\eta R\lambda\sqrt{\Gamma\ell}
       +24\eta R\lambda\ell\notag\\
 &=\frac{R^2}{K}
       \left(8\sqrt{\frac{\ell}{\Gamma}}+
                    24\frac{\ell}{\Gamma}\right)
 \le\frac{32R^2}{K}.
 \label{eq:corrupt-radius-martingale}
\end{align}
The last two steps use $\eta=R/(K\lambda\Gamma)$ and $\ell\le\Gamma$.

For the comparison potential, let
$A_t=\eta_tv_{t-1}(z_{t-1}-x_t)$, as in Lemma~\ref{lem:prefix}.
This vector is $\F_{t-1}$-measurable. On $\{I_t=1\}$,
the definition $P_{t-1}=v_{t-1}\norm{z_{t-1}-x_t}^2/2$
and Lemma~\ref{lem:weights} give
\begin{align}
 \norm{A_t}
 &=\eta_tv_{t-1}\norm{z_{t-1}-x_t}
 =\eta_t\sqrt{v_{t-1}}\sqrt{2P_{t-1}}\notag\\
 &\le\frac{\sqrt2\eta}{T}
       \sqrt{2\cdot\frac{4R^2}{T^2}}
 =\frac{4\eta R}{T^2}.
 \label{eq:corrupt-A}
\end{align}
Thus $M_s^P=\sum_{t=1}^s I_t\ip{\xi_t}{A_t}$, with $M_0^P=0$,
satisfies
\begin{align*}
 \Et(M_t^P-M_{t-1}^P)
 &=I_t\ip{\Et\xi_t}{A_t}=0,\\
 |M_t^P-M_{t-1}^P|
 &\le I_t\norm{\xi_t}\norm{A_t}
 \le(3\lambda)\frac{4\eta R}{T^2}
 =\frac{12\eta R\lambda}{T^2},\\
 \sum_{t=1}^T\Et(M_t^P-M_{t-1}^P)^2
 &\le\sum_{t=1}^T I_t\norm{A_t}^2\Et\norm{\xi_t}^2\\
 &\le\frac{16\eta^2R^2VT}{T^4}
 \le\frac{32\eta^2R^2\lambda^2\Gamma}{T^4}.
\end{align*}
Applying Lemma~\ref{lem:freedman} again gives an event $E_P$,
of probability at least $1-e^{-\ell}$, on which
\begin{align}
 \max_{s\le T}M_s^P
 &\le\sqrt{\frac{32\eta^2R^2VT\ell}{T^4}}
       +\frac{24\eta R\lambda\ell}{T^2}\notag\\
 &\le\frac{8\eta R\lambda\sqrt{\Gamma\ell}
                  +24\eta R\lambda\ell}{T^2}\notag\\
 &=\frac{R^2}{KT^2}
       \left(8\sqrt{\frac{\ell}{\Gamma}}+
                    24\frac{\ell}{\Gamma}\right)
 \le\frac{32R^2}{KT^2}.
 \label{eq:corrupt-potential-martingale}
\end{align}

The remaining inner products require no concentration inequality.
By Cauchy--Schwarz and \eqref{eq:corrupt-nonmartingale-budget},
\begin{align}
 2\sum_{t=1}^T I_t\eta_t
       |\ip{b_t+d_t}{x_t-x^\star}|
 &\le2\eta(2R)\sum_{t=1}^T I_t(\norm{b_t}+\norm{d_t})\notag\\
 &\le4\eta R(4\lambda\Gamma)
 =\frac{16R^2}{K},
 \label{eq:corrupt-radius-bias}\\
 \sum_{t=1}^T I_t|\ip{b_t+d_t}{A_t}|
 &\le\frac{4\eta R}{T^2}
            \sum_{t=1}^T I_t(\norm{b_t}+\norm{d_t})\notag\\
 &\le\frac{4\eta R}{T^2}(4\lambda\Gamma)
 =\frac{16R^2}{KT^2}.
 \label{eq:corrupt-potential-bias}
\end{align}

\subsection{Removing the stopping time and bounding the last iterate}
Let $E=E_Q\cap E_r\cap E_P$. The union bound gives
\[
 \Prob(E)\ge1-\Prob(E_Q^c)-\Prob(E_r^c)-\Prob(E_P^c)
 \ge1-3e^{-\ell}=1-\frac{3\delta}{8}\ge1-\delta.
\]
Fix a sample path in $E$ and an integer $1\le s\le T$ with $s<\tau$.
Then $I_t=1$ for $1\le t\le s$. Lemma~\ref{lem:radius} gives
\begin{align*}
 \norm{x_{t+1}-x^\star}^2-\norm{x_t-x^\star}^2
 &\le-2\eta_t\ip{e_t}{x_t-x^\star}+2\eta_t^2\norm{e_t}^2\\
 &=-2\eta_t\ip{\xi_t}{x_t-x^\star}
    -2\eta_t\ip{b_t+d_t}{x_t-x^\star}
    +2\eta_t^2\norm{e_t}^2.
\end{align*}
Sum these inequalities over $t=1,\ldots,s$, and use
\eqref{eq:corrupt-Q}, \eqref{eq:corrupt-radius-martingale}, and
\eqref{eq:corrupt-radius-bias}:
\begin{align}
 \norm{x_{s+1}-x^\star}^2
 &\le\norm{x_1-x^\star}^2+M_s^r
      -2\sum_{t=1}^s\eta_t\ip{b_t+d_t}{x_t-x^\star}
      +2\sum_{t=1}^s\eta_t^2\norm{e_t}^2\notag\\
 &\le R^2+\frac{32R^2}{K}+\frac{16R^2}{K}
      +2\eta^2Q_e\notag\\
 &\le R^2+\frac{48R^2}{K}
      +2\frac{R^2}{K^2\lambda^2\Gamma^2}
                  (7\lambda^2\Gamma)\notag\\
 &=\left(1+\frac{48}{K}+\frac{14}{K^2\Gamma}\right)R^2
 <2R^2.
 \label{eq:corrupt-close-radius}
\end{align}

For the objective gap, \eqref{eq:prox-smooth} with $r=0$ implies
\eqref{eq:abstractstep} with
$\rho_t=2\eta_t^2\norm{e_t}^2$; the factor $2$ only enlarges its
nonnegative error term. Lemma~\ref{lem:prefix} consequently gives
\begin{align*}
 &P_s+\eta_sv_s\Delta_{s+1}
    +S_{s+1}\sum_{j=2}^s(v_j-v_{j-1})\Delta_j\\
 &\qquad\le
 P_0+2\sum_{t=1}^s\eta_t^2v_t\norm{e_t}^2
       +\sum_{t=1}^s\ip{\xi_t+b_t+d_t}{A_t}.
\end{align*}
The third term on the left is nonnegative. Since
$2\eta_t^2v_t\le4\eta^2/T^2$ by Lemma~\ref{lem:weights}, we get
\begin{align}
 P_s+\eta_sv_s\Delta_{s+1}
 &\le P_0+M_s^P+\sum_{t=1}^s\ip{b_t+d_t}{A_t}
             +2\sum_{t=1}^s\eta_t^2v_t\norm{e_t}^2\notag\\
 &\le\frac{R^2}{T^2}+\frac{32R^2}{KT^2}
       +\frac{16R^2}{KT^2}+\frac{4\eta^2Q_e}{T^2}\notag\\
 &\le\frac{R^2}{T^2}+\frac{48R^2}{KT^2}
       +\frac4{T^2}\frac{R^2}{K^2\lambda^2\Gamma^2}
                          (7\lambda^2\Gamma)\notag\\
 &=\left(1+\frac{48}{K}+\frac{28}{K^2\Gamma}\right)
                      \frac{R^2}{T^2}
 <\frac{2R^2}{T^2}.
 \label{eq:corrupt-close-potential}
\end{align}
Both strict inequalities follow from
\[
 1+\frac{48}{K}+\frac{28}{K^2\Gamma}
 <1+\frac{48}{128}+\frac{28}{128^2\log8}
 <1.376<2.
\]

If $\tau<\infty$, then $2\le\tau\le T+1$. Applying
\eqref{eq:corrupt-close-radius}--\eqref{eq:corrupt-close-potential}
with $s=\tau-1$ yields
\[
 \norm{x_\tau-x^\star}^2<2R^2<4R^2,\qquad
 P_{\tau-1}\le
 P_{\tau-1}+\eta_{\tau-1}v_{\tau-1}\Delta_\tau
 <\frac{2R^2}{T^2}<\frac{4R^2}{T^2}.
\]
Neither stopping condition can hold, a contradiction. Therefore
$\tau=\infty$ on $E$, and $\norm{x_t-x^\star}\le2R$ for
$1\le t\le T+1$.

Finally, take $s=T$ in \eqref{eq:corrupt-close-potential}. Since
$P_T\ge0$, $v_T=1$, and $\eta_T=\eta/T$,
\begin{align*}
 \frac{\eta}{T}\bigl(f(x_{T+1})-f^\star\bigr)
 &\le P_T+\eta_Tv_T\Delta_{T+1}
 \le\frac{2R^2}{T^2},\\
 f(x_{T+1})-f^\star
 &\le\frac{2R^2}{\eta T}
 =\frac{2R^2}{T}\frac{K\lambda\Gamma}{R}
 =\frac{256R\lambda\Gamma}{T}\\
 &\le\frac{256R\Gamma}{T}
          \left(4LR+\sigma(T/\Gamma)^{1/p}\right)\\
 &=1024LR^2\frac{\Gamma}{T}
       +256\sigma R\left(\frac{\Gamma}{T}\right)^{(p-1)/p}.
\end{align*}
Together with $\Prob(E)\ge1-\delta$ and $\Gamma=\ell+q$, this proves
Theorem~\ref{thm:corrupt}.\hfill$\square$

\subsection{Proof of Corollary~\ref{cor:budget}}
Keep the same filtration, stopping time, and error decomposition.
Now $\nu_t$ is predictable, $\sum_{t=1}^T\nu_t\le S$ almost surely, and
$\lambda=\max\{4LR,(S/\Gamma)^{1/p}\}$. In particular,
$S\le\lambda^p\Gamma$. The conditional version of
Lemma~\ref{lem:clip} gives, on $\{I_t=1\}$,
\[
 \norm{b_t}\le B_t:=2\nu_t\lambda^{1-p},\qquad
 \Et\norm{u_t}^2\le V_t:=2\nu_t\lambda^{2-p},\qquad
 \Et\norm{\xi_t}^2\le V_t.
\]
Thus
\begin{align*}
 \sum_{t=1}^T I_tB_t
 &=2\lambda^{1-p}\sum_{t=1}^T I_t\nu_t
 \le2S\lambda^{1-p}
 \le2\lambda^p\Gamma\lambda^{1-p}=2\lambda\Gamma,\\
 \sum_{t=1}^T I_tV_t
 &=2\lambda^{2-p}\sum_{t=1}^T I_t\nu_t
 \le2S\lambda^{2-p}
 \le2\lambda^p\Gamma\lambda^{2-p}=2\lambda^2\Gamma.
\end{align*}
All pointwise clipping bounds are unchanged. Apply
Lemma~\ref{lem:nonnegative} to the sequence $I_t\norm{u_t}^2$ with
$H=3\lambda^2$, $m_t=I_tV_t$, $A=2S\lambda^{2-p}$, and $u=\ell$.
Except on an event of probability $e^{-\ell}$,
\begin{align*}
 Q_e
 &\le\sum_{t=1}^T I_t\norm{u_t}^2+3\lambda^2q\\
 &\le4S\lambda^{2-p}+3\lambda^2\ell+3\lambda^2q\\
 &\le4\lambda^2\Gamma+3\lambda^2(\ell+q)
 =7\lambda^2\Gamma.
\end{align*}
The other inputs to the proof satisfy
\begin{align*}
 \sum_{t=1}^T I_t(\norm{b_t}+\norm{d_t})
 &\le\sum_{t=1}^T I_tB_t+2\lambda q
 \le2\lambda\Gamma+2\lambda q\le4\lambda\Gamma,\\
 \sum_{t=1}^T\Et(M_t^r-M_{t-1}^r)^2
 &\le16\eta^2R^2\sum_{t=1}^T I_tV_t
 \le32\eta^2R^2\lambda^2\Gamma,\\
 \sum_{t=1}^T\Et(M_t^P-M_{t-1}^P)^2
 &\le\frac{16\eta^2R^2}{T^4}\sum_{t=1}^T I_tV_t
 \le\frac{32\eta^2R^2\lambda^2\Gamma}{T^4}.
\end{align*}
The increment bounds remain $12\eta R\lambda$ and
$12\eta R\lambda/T^2$. Hence the two applications of
Lemma~\ref{lem:freedman} give exactly
\eqref{eq:corrupt-radius-martingale} and
\eqref{eq:corrupt-potential-martingale}. The first of the three bounds above
also gives \eqref{eq:corrupt-radius-bias} and
\eqref{eq:corrupt-potential-bias}.

On the intersection of these three events, the calculations in
\eqref{eq:corrupt-close-radius}--\eqref{eq:corrupt-close-potential}
therefore yield
\[
 \norm{x_{s+1}-x^\star}^2<2R^2,\qquad
 P_s+\eta_sv_s\Delta_{s+1}<2R^2/T^2
 \quad(1\le s\le T,\ s<\tau).
\]
If $\tau<\infty$, setting $s=\tau-1$ contradicts both possible
stopping conditions as above. Thus $\tau=\infty$ and, with probability
at least $1-3e^{-\ell}\ge1-\delta$,
\begin{align*}
 f(x_{T+1})-f^\star
 &\le\frac{256R\lambda\Gamma}{T}\\
 &\le\frac{256R\Gamma}{T}
           \left(4LR+(S/\Gamma)^{1/p}\right)\\
 &=1024\frac{LR^2\Gamma}{T}
       +256\frac{RS^{1/p}\Gamma^{(p-1)/p}}{T}.
\end{align*}
The almost-sure bound $\sum_t\nu_t\le S$ supplies the deterministic
variance bounds required by the concentration lemmas.\hfill$\square$

\section{Proofs of the statistical lower bounds}\label{app:lower}
Each proof constructs two instances indexed by $s\in\{-1,+1\}$.
To compare their outputs, run the algorithm on both instances with the
same internal random seed and the same Bernoulli variables. Each run
still has the distribution required by its own instance. We will identify
an event on which the observed transcripts agree, and then show that a
common output must have a large error on at least one instance.

\subsection{Proof of Proposition~\ref{prop:lower}}
Let $a=(p-1)/p$ and $b=\log(1/(4\delta))$. Choose
\begin{equation}\label{eq:lower-smooth-parameters}
 \rho=\frac{b}{2T},\qquad
 m=\min\left\{LR,\frac{\sigma\rho^a}{2^{1/p}}\right\}
   =\min\left\{LR,\frac{\sigma}{2}(b/T)^a\right\}.
\end{equation}
Indeed, $2^{-1/p}\rho^a=2^{-1/p-a}(b/T)^a=(b/T)^a/2$.
The assumptions $\delta\le1/8$ and $T\ge b$ imply
\[
 b\ge\log2>0,\qquad 0<\rho=\frac{b}{2T}\le\frac12,\qquad m>0.
\]
For $s\in\{-1,+1\}$, define
\[
 f_s(x)=\frac{m}{2R}(x-sR)^2,\qquad x\in\R.
\]
These functions satisfy
\begin{align*}
 f_s'(x)&=\frac{m}{R}(x-sR)=\frac{m}{R}x-sm,\qquad
 f_s''(x)=\frac{m}{R}>0,\\
 |f_s'(x)-f_s'(y)|
 &=\frac{m}{R}|x-y|\le L|x-y|,\\
 x_s^\star&=sR,\qquad f_s(x_s^\star)=0,\qquad
 |x_1-x_s^\star|=|0-sR|=R.
\end{align*}
Thus both instances are convex and $L$-smooth with the required
initial distance.

At the $t$-th query draw a fresh $B_t\sim\operatorname{Bernoulli}(\rho)$,
independently of all preceding samples and the algorithm's random seed.
Return
\begin{equation}\label{eq:lower-smooth-oracle}
 g_{s,t}(x)=\frac{m}{R}x-\frac{sm}{\rho}B_t.
\end{equation}
For an adaptive query $x_t$, include the random seed in $\F_{t-1}$;
then $x_t$ is $\F_{t-1}$-measurable and $\Et B_t=\rho$. Therefore
\begin{align*}
 \Et g_{s,t}(x_t)
 &=\frac{m}{R}x_t-\frac{sm}{\rho}\Et B_t
 =\frac{m}{R}x_t-sm=f_s'(x_t),\\
 g_{s,t}(x_t)-f_s'(x_t)
 &=sm\left(1-\frac{B_t}{\rho}\right).
\end{align*}
Conditionally on $\F_{t-1}$, the magnitude of the last expression is
$m$ with probability $1-\rho$, and $m(1-\rho)/\rho$ with probability
$\rho$. Hence
\begin{align}
 \Et|g_{s,t}(x_t)-f_s'(x_t)|^p
 &=m^p\left[(1-\rho)
             +\rho\left(\frac{1-\rho}{\rho}\right)^p\right]\notag\\
 &=m^p\left[(1-\rho)+\rho^{1-p}(1-\rho)^p\right]\notag\\
 &\le m^p\left[\rho^{1-p}+\rho^{1-p}\right]
 =2m^p\rho^{1-p}\notag\\
 &\le2\left(\frac{\sigma\rho^a}{2^{1/p}}\right)^p\rho^{1-p}
 =\sigma^p\rho^{ap+1-p}
 =\sigma^p.
 \label{eq:lower-moment}
\end{align}
We used $1-\rho\le1\le\rho^{1-p}$, $(1-\rho)^p\le1$, and $ap=p-1$.
The bound does not depend on the query location.

Consider the event $E_0=\{B_1=\cdots=B_T=0\}$. For $0\le u\le1/2$,
\[
 \frac{d}{du}\bigl(\log(1-u)+2u\bigr)
 =-\frac{1}{1-u}+2=\frac{1-2u}{1-u}\ge0.
\]
Since $\log(1-0)+2\cdot0=0$, this proves $\log(1-u)\ge-2u$.
Independence of the Bernoulli variables now gives
\begin{align}
 \Prob(E_0)
 &=\prod_{t=1}^T\Prob(B_t=0)
 =(1-\rho)^T
 =\exp\bigl(T\log(1-\rho)\bigr)\notag\\
 &\ge\exp(-2T\rho)
 =\exp(-b)
 =4\delta.
 \label{eq:lower-empty-prob}
\end{align}

On $E_0$, the answer to every query is $(m/R)x$, for either sign.
Fixing the common random seed, the first queries agree. If the
transcripts through round $t-1$ agree, the algorithm selects the same
next query; the responses $(m/R)x_t$ then agree as well. This induction
also gives the same stopping decision and output. If the algorithm uses
fewer than $T$ queries, generate unused Bernoulli variables up to $T$;
this leaves its output unchanged and keeps
\eqref{eq:lower-empty-prob} valid.

Let $\widehat x_s$ denote the output on instance $s$, and let
$\widehat x$ be their common value on $E_0$. At that value,
\begin{align*}
 &\bigl(f_+(\widehat x)-f_+(R)\bigr)
   +\bigl(f_-(\widehat x)-f_-(-R)\bigr)\\
 &\qquad=\frac{m}{2R}
            \left((\widehat x-R)^2+(\widehat x+R)^2\right)
 =\frac{m}{R}(\widehat x^2+R^2)\ge mR.
\end{align*}
At least one of the two summands must therefore be at least $mR/2$.
Pointwise under the coupling,
\[
 \ind_{E_0}
 \le
 \ind_{\{f_+(\widehat x_+)-f_+(R)\ge mR/2\}}
 +
 \ind_{\{f_-(\widehat x_-)-f_-(-R)\ge mR/2\}}.
\]
Taking expectations and then the larger of the two probabilities gives
\begin{align*}
 \sum_{s\in\{-1,+1\}}
 \Prob\{f_s(\widehat x_s)-f_s(sR)\ge mR/2\}
 &\ge\Prob(E_0)\ge4\delta,\\
 \max_{s\in\{-1,+1\}}
 \Prob\{f_s(\widehat x_s)-f_s(sR)\ge mR/2\}
 &\ge\frac12\Prob(E_0)\ge2\delta.
\end{align*}
The maximum selects one fixed instance; the sign need not depend on
the realized transcript. Finally,
\[
 \frac{mR}{2}
 =\frac{R}{2}\min\left\{LR,\frac{\sigma}{2}(b/T)^a\right\}
 =\min\left\{\frac{LR^2}{2},
                \frac{\sigma R}{4}(b/T)^{(p-1)/p}\right\}.
\]
This proves Proposition~\ref{prop:lower}.\hfill$\square$

\subsection{Proof of Proposition~\ref{prop:nonsmoothlower}}
Use the same $a=(p-1)/p$, $b=\log(1/(4\delta))$, and $\rho=b/(2T)$,
but now set
\[
 m=\min\left\{G,\frac{\sigma}{2}(b/T)^a\right\}
  =\min\left\{G,\frac{\sigma\rho^a}{2^{1/p}}\right\},\qquad
 f_s(x)=-smx,\qquad r(x)=\iota_{[-R,R]}(x).
\]
The linear function $f_s$ is convex and satisfies
\[
 |f_s(x)-f_s(y)|=m|x-y|\le G|x-y|,\qquad
 \partial f_s(x)=\{-sm\}.
\]
The indicator $r$ of the closed convex interval is proper, lower
semicontinuous, and convex. Direct minimization on this interval gives
\[
 x_s^\star=sR,\qquad
 F_s^\star=f_s(sR)=-mR,\qquad |x_1-x_s^\star|=R.
\]
For every $\alpha>0$, its exact proximal map is
\begin{align*}
 \prox_{\alpha r}(y)
 &=\arg\min_{x\in\R}
          \left\{r(x)+\frac{|x-y|^2}{2\alpha}\right\}\\
 &=\arg\min_{x\in[-R,R]}|x-y|^2
 =\begin{cases}
     -R,&y<-R,\\
     y,&-R\le y\le R,\\
     R,&y>R.
   \end{cases}
\end{align*}
This oracle is identical for the two signs, for every query $y,\alpha$.

For each stochastic query, draw a fresh
$B_t\sim\operatorname{Bernoulli}(\rho)$ and return
$g_{s,t}(x)=-smB_t/\rho$. For any adaptive query,
\begin{align*}
 \Et g_{s,t}(x_t)
 &=-\frac{sm}{\rho}\Et B_t=-sm\in\partial f_s(x_t),\\
 \Et|g_{s,t}(x_t)+sm|^p
 &=m^p\left[(1-\rho)+\rho^{1-p}(1-\rho)^p\right]\\
 &\le2m^p\rho^{1-p}\\
 &\le2\left(\frac{\sigma\rho^a}{2^{1/p}}\right)^p\rho^{1-p}
 =\sigma^p.
\end{align*}
The moment calculation uses the same two Bernoulli outcomes as
\eqref{eq:lower-moment}.

On $E_0=\{B_1=\cdots=B_T=0\}$, every stochastic reply is zero.
Fix the same random seed in both runs. Induction over all oracle calls,
including any intervening proximal calls, gives identical transcripts:
stochastic answers are zero, and proximal answers come from the same
map displayed above. The outputs therefore agree on $E_0$. As before,
unused Bernoulli variables may be generated after early stopping, and
$\Prob(E_0)\ge4\delta$ by \eqref{eq:lower-empty-prob}.

If the common output $\widehat x$ belongs to $[-R,R]$, then
\begin{align*}
 F_+(\widehat x)-F_+^\star
 &=-m\widehat x-(-mR)=m(R-\widehat x),\\
 F_-(\widehat x)-F_-^\star
 &=m\widehat x-(-mR)=m(R+\widehat x),\\
 \bigl(F_+(\widehat x)-F_+^\star\bigr)
 +\bigl(F_-(\widehat x)-F_-^\star\bigr)&=2mR.
\end{align*}
At least one gap is at least $mR$. If $\widehat x\notin[-R,R]$,
both gaps equal $+\infty$, so the same conclusion holds. Consequently,
\begin{align*}
 \ind_{E_0}
 &\le\sum_{s\in\{-1,+1\}}
       \ind_{\{F_s(\widehat x_s)-F_s^\star\ge mR\}},\\
 \sum_{s\in\{-1,+1\}}
 \Prob\{F_s(\widehat x_s)-F_s^\star\ge mR\}
 &\ge\Prob(E_0)\ge4\delta,\\
 \max_{s\in\{-1,+1\}}
 \Prob\{F_s(\widehat x_s)-F_s^\star\ge mR\}
 &\ge\frac12\Prob(E_0)\ge2\delta.
\end{align*}
Substituting
\[
 mR=\min\left\{GR,\frac{\sigma R}{2}(b/T)^{(p-1)/p}\right\}
\]
proves Proposition~\ref{prop:nonsmoothlower}.\hfill$\square$

\subsection{Proof of Proposition~\ref{prop:corruptlower}}
Choose
\[
 a=\frac{p-1}{p},\qquad
 \rho=\frac{q}{2T},\qquad
 m=\min\{LR,(\sigma/2)\rho^a\}.
\]
Since $1\le q\le T$, we have $0<\rho\le1/2$ and $m>0$.
Use the quadratic functions and latent clean oracle
\[
 f_s(x)=\frac{m}{2R}(x-sR)^2,\qquad
 g_{s,t}(x)=\frac{m}{R}x-\frac{sm}{\rho}B_t,
 \qquad B_t\sim\operatorname{Bernoulli}(\rho).
\]
The calculations in the first proof show that the functions are
convex and $L$-smooth, that $x_s^\star=sR$, and that
$|x_1-x_s^\star|=R$. Fresh Bernoulli samples also give
\begin{align*}
 \Et g_{s,t}(x_t)
 &=\frac{m}{R}x_t-sm=f_s'(x_t),\\
 \Et|g_{s,t}(x_t)-f_s'(x_t)|^p
 &=m^p\left[(1-\rho)+\rho^{1-p}(1-\rho)^p\right]\\
 &\le2m^p\rho^{1-p}\\
 &\le2\left(\frac{\sigma}{2}\rho^a\right)^p\rho^{1-p}
 =2^{1-p}\sigma^p\le\sigma^p.
\end{align*}
These are conditional statements with respect to the enlarged
filtration containing past latent samples and past adversarial choices.

For each instance, the adversary replaces an informative sample
($B_t=1$) if fewer than $q$ replacements have already been made.
Equivalently, define recursively
\[
 C_t=B_t\ind_{\{\sum_{j=1}^{t-1}C_j<q\}},\qquad
 \widetilde g_{s,t}(x_t)=
 \begin{cases}
   (m/R)x_t,&C_t=1,\\
   g_{s,t}(x_t),&C_t=0.
 \end{cases}
\]
When $C_t=1$, the replacement differs from the latent reply by
$sm/\rho\ne0$, so $C_t$ is exactly the replacement indicator.
The rule is causal: from the current query and sample the adversary
can distinguish
\[
 B_t=0\ \Longrightarrow\ g_{s,t}(x_t)=(m/R)x_t,\qquad
 B_t=1\ \Longrightarrow\ g_{s,t}(x_t)=(m/R)x_t-sm/\rho.
\]
Once the running count reaches $q$, all later $C_t$ are zero.
Therefore $\sum_{t=1}^T C_t\le q$ on every sample path.

Let $N_B=\sum_{t=1}^T B_t$ and $E_q=\{N_B\le q\}$. By linearity of
expectation and Markov's inequality,
\begin{align*}
 \E N_B&=\sum_{t=1}^T\E B_t=T\rho=q/2,\\
 \Prob(N_B>q)
 &\le\Prob(N_B\ge q)\le\frac{\E N_B}{q}=\frac12,\\
 \Prob(E_q)&=1-\Prob(N_B>q)\ge\frac12.
\end{align*}
On $E_q$, whenever $B_t=1$ there have been at most $q-1$ earlier
informative samples, so the replacement budget has not yet been
exhausted. Thus every informative sample is replaced, and
\[
 \widetilde g_{s,t}(x_t)=\frac{m}{R}x_t
 \qquad\text{for all }t\le T\text{ on }E_q.
\]
With the same random seed and Bernoulli sequence in the two runs,
these responses give identical adaptive query transcripts and hence
identical outputs on $E_q$. For early-stopping algorithms we again
generate the unused Bernoulli variables without making further queries.

For the common output $\widehat x$ on $E_q$,
\begin{align*}
 &\bigl(f_+(\widehat x)-f_+(R)\bigr)
   +\bigl(f_-(\widehat x)-f_-(-R)\bigr)\\
 &\qquad=\frac{m}{2R}
          \left((\widehat x-R)^2+(\widehat x+R)^2\right)
 =\frac{m}{R}(\widehat x^2+R^2)\ge mR.
\end{align*}
It follows that
\begin{align*}
 \ind_{E_q}
 &\le\sum_{s\in\{-1,+1\}}
    \ind_{\{f_s(\widehat x_s)-f_s(sR)\ge mR/2\}},\\
 \sum_{s\in\{-1,+1\}}
 \Prob\{f_s(\widehat x_s)-f_s(sR)\ge mR/2\}
 &\ge\Prob(E_q)\ge\frac12,\\
 \max_{s\in\{-1,+1\}}
 \Prob\{f_s(\widehat x_s)-f_s(sR)\ge mR/2\}
 &\ge\frac12\Prob(E_q)\ge\frac14.
\end{align*}
Each probability is the marginal law of one clean instance together
with its own causal adversary. Hence the maximizing sign gives one
fixed admissible instance and adversary, with threshold
\[
 \frac{mR}{2}
 =\frac{R}{2}\min\left\{LR,\frac{\sigma}{2}
                  \left(\frac{q}{2T}\right)^{(p-1)/p}\right\}.
\]
This proves Proposition~\ref{prop:corruptlower}.\hfill$\square$

\end{document}